\documentclass{article}

\usepackage{arxiv}

\usepackage{times}  
\usepackage{helvet}  
\usepackage{courier}  
\usepackage[hyphens]{url}  
\usepackage{graphicx} 
\usepackage{caption} 
\usepackage{algorithm}
\usepackage{algorithmic}

\usepackage{amsthm}
\usepackage{amsmath}
\usepackage{amsfonts}
\usepackage{subfigure}
\usepackage{multirow}
\usepackage{array}
\usepackage{booktabs}

\newtheorem{thm}{Theorem}
\newtheorem{apdx_thm}{Theorem}
\newtheorem{lemma}{Lemma}
\newtheorem{apdx_lemma}{Lemma}

\usepackage{newfloat}
\usepackage{listings}

\title{Model-based Reinforcement Learning with Multi-step Plan Value Estimation}
\author{%
  Haoxin Lin\thanks{Equal Contribution\quad $^\dagger$Corresponding Author},~~
    Yihao Sun\footnotemark[1],~~
    Jiaji Zhang,~~
    Yang Yu\textsuperscript{$\dagger$} \\   National Key Laboratory for Novel Software Technology, Nanjing University, Nanjing, Jiangsu, China\\
    \{linhx, sunyh, zhangjj\}@lamda.nju.edu.cn, yuy@nju.edu.cn
}
\date{}

\begin{document}

\maketitle

\begin{abstract}
A promising way to improve the sample efficiency of reinforcement learning is model-based methods, in which many explorations and evaluations can happen in the learned models to save real-world samples. However, when the learned model has a non-negligible model error, sequential steps in the model are hard to be accurately evaluated, limiting the model's utilization. This paper proposes to alleviate this issue by introducing multi-step plans to replace multi-step actions for model-based RL. We employ the multi-step plan value estimation, which evaluates the expected discounted return after executing a sequence of action plans at a given state, and updates the policy by directly computing the multi-step policy gradient via plan value estimation. The new model-based reinforcement learning algorithm MPPVE (\textbf{M}odel-based Planning \textbf{P}olicy Learning with Multi-step \textbf{P}lan \textbf{V}alue \textbf{E}stimation) shows a better utilization of the learned model and achieves a better sample efficiency than state-of-the-art model-based RL approaches.
\end{abstract}

\section{Introduction}
Reinforcement Learning (RL) has attracted close attention in recent years. Despite its empirical success in simulated domains and games, insufficient sample efficiency is still one of the critical problems hindering the application of RL in reality \cite{yu2018sample}. One promising way to improve the sample efficiency is to train and utilize world models \cite{yu2018sample}, which is known as model-based RL (MBRL). World model learning has recently received significant developments, including the elimination of the compounding error issue \cite{xu2021error} and causal model learning studies \cite{chen2022arxiv,zhu2022arxiv} for achieving high-fidelity models, and has also gained real-world applications \cite{mail,shang2021mlj}. Nevertheless, this paper focus on the dyna-style MBRL framework \cite{dyna-q} that augments the replay buffer by the world model generated data for off-policy reinforcement learning.

In dyna-style MBRL, the model is often learned by supervised learning to fit the observed transition data, which is simple to train but exhibits non-negligible model error \cite{xu2021error}. Specifically, consider a generated $k$-step model rollout $s_t,a_t,\hat{s}_{t+1},a_{t+1},...,\hat{s}_{t+k}$, where $\hat{s}$ stands for fake states. The deviation error of fake states increases with $k$ since the error accumulates gradually as the state transitions in imagination. If updated on the fake states with a large deviation, the policy will be misled by policy gradients given by biased state-action value estimation. The influence of model error on the directions of policy gradients is demonstrated in Figure \ref{cos_sim_of_pg}.

Therefore, the dyna-style MBRL often introduces techniques to reduce the impact of the model error. For example, MBPO \cite{mbpo} proposes the branched rollouts scheme with a gradually growing branch length to truncate imaginary model rollouts, avoiding the participation of unreliable fake samples in policy optimization. BMPO \cite{bmpo} further truncates model rollouts into a shorter branch length than MBPO after the same number of steps of environmental sampling since it can make imaginary rollouts from both forward and backward directions with the bidirectional dynamics model.

\begin{figure}[h]
    \centering
    \includegraphics[width=0.75\linewidth]{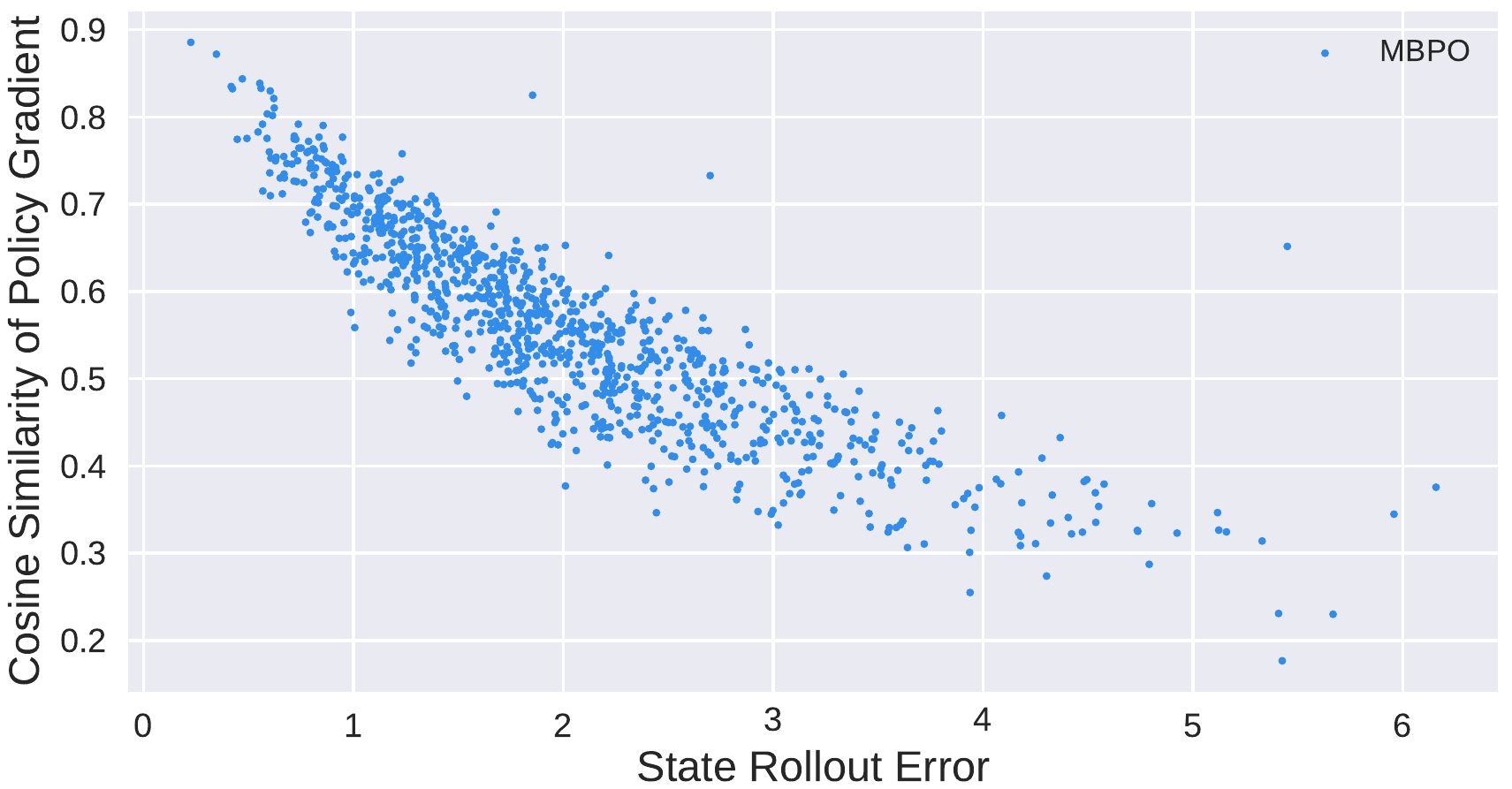}
    \caption{A simulated experiment that reveals the influence of model error on the directions of policy gradients. We make real branched rollouts in the real world and utilize the learned model to generate the fake, starting from some real states. For each fake rollout and its corresponding real rollout, we show their deviation along with the normalized cosine similarity between their multi-step policy gradients.}
    \label{cos_sim_of_pg}
\end{figure}

\begin{figure*}[t!]
\centering
\subfigure[Data generation. Environmental data is stored in $\mathcal{D}_{\mathrm{env}}$, while model data is stored in $\mathcal{D}_{\mathrm{model}}$.]{
    \centering
    \includegraphics[height=0.087\linewidth]{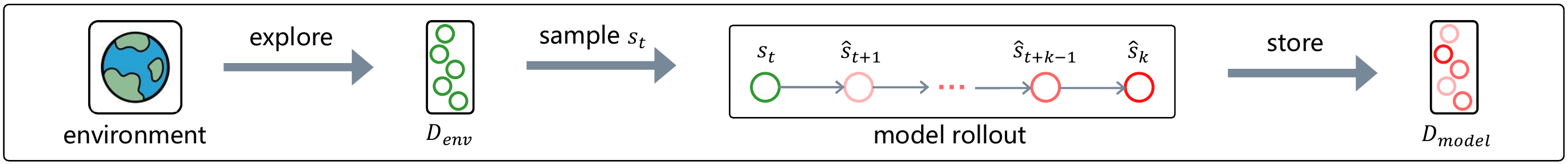}
    \label{fig0(a)}
}\\
\subfigure[Update actor with action-value.]{
    \centering
    \includegraphics[height=0.2\linewidth]{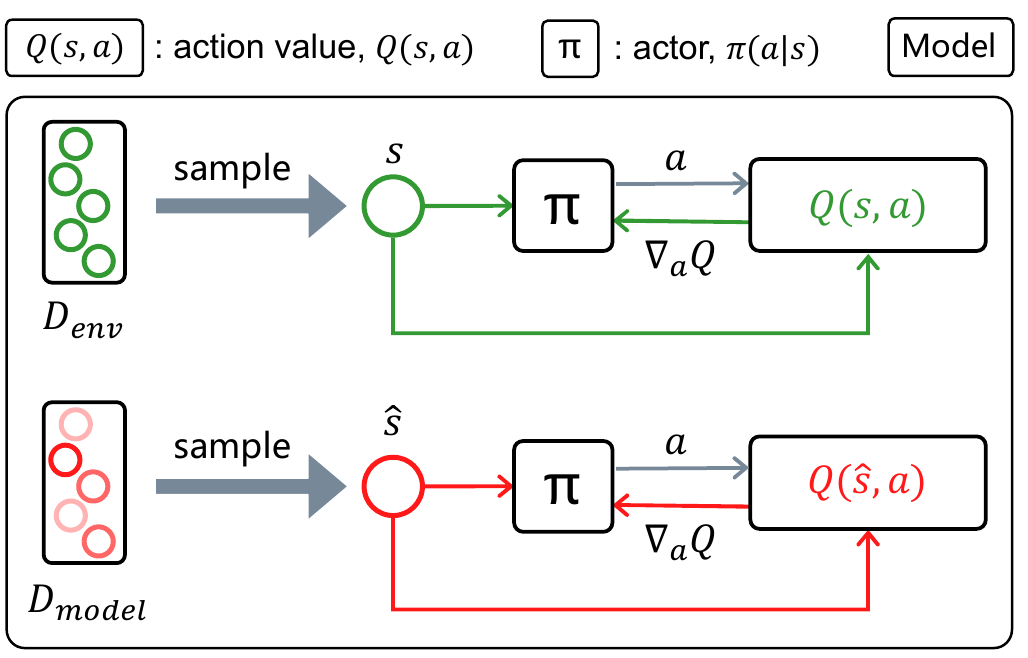}
    \label{fig0(b)}
}%
\hspace{-0.2cm}
\subfigure[Update actor with multi-step plan value.]{
    \centering
    \includegraphics[height=0.2\linewidth]{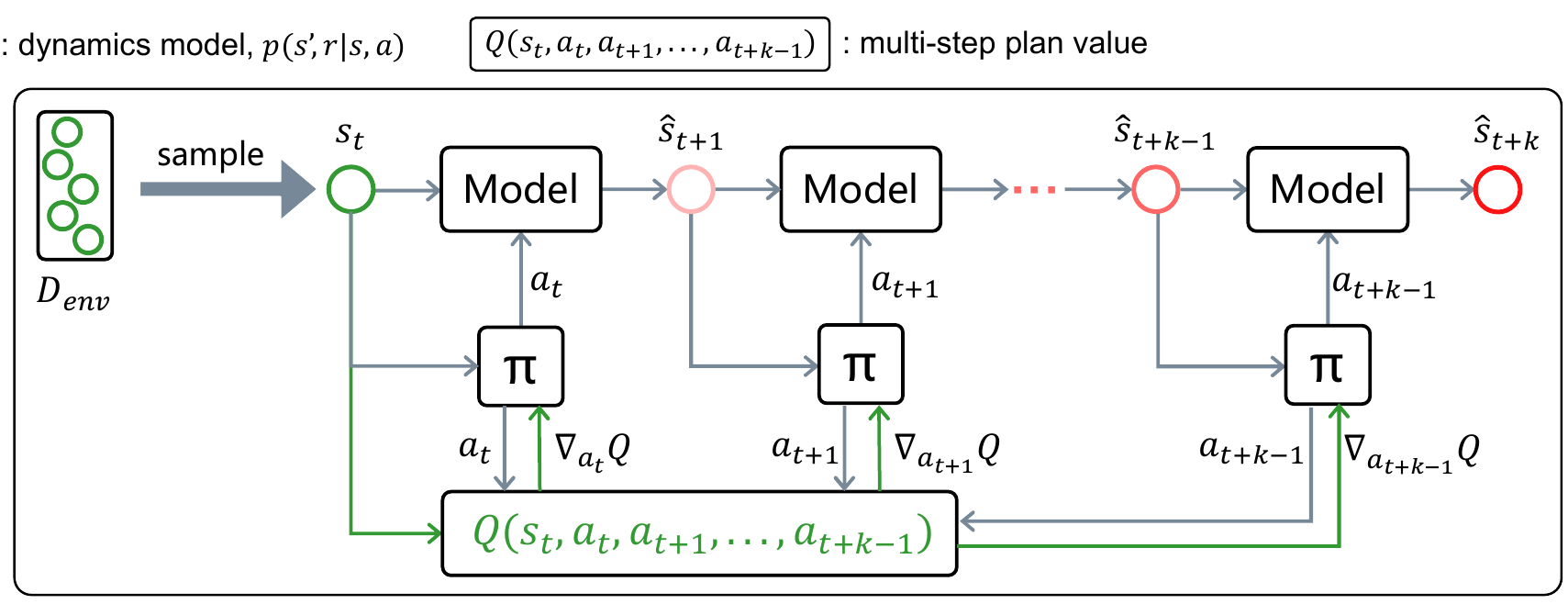}
    \label{fig0(c)}
}%
\centering
\caption{A comparison between MPPVE and previous model-based methods in updating actor. (a) Data generation, where green $s_t$ stands for the real environmental state, red $\hat{s}_{t+1},...,\hat{s}_{t+k}$ stand for fake states, and the darker the red becomes, the greater the deviation error is in expectation. (b) Illustration of how previous model-based methods update actor with action-value estimation. (c) Illustration of how MPPVE updates actor with $k$-step plan value estimation. The $k$-step policy gradients are given by plan value estimation directly when the actor plans starting from the real environmental state.
}
\label{fig0}
\end{figure*}

We argue that, while the above previous methods tried to reduce the impact of the accumulated model error by reducing the rollout steps, avoiding the explicit reliance on the rolled-out fake states can be helpful. This paper proposes employing $k$-step plan value estimation $Q(s_t,a_t,...,a_{t+k-1})$ to evaluate sequential action plans. When using the $k$-step rollout data to update the policy, we can compute the policy gradient only on the starting real state $s_t$, but not on any fake states. Therefore, the $k$-step policy gradients directly given by plan value are less influenced by the model error. Based on this key idea, we formulate a new model-based algorithm called \textbf{M}odel-based Planning \textbf{P}olicy Learning with Multi-step \textbf{P}lan \textbf{V}alue \textbf{E}stimation (MPPVE). The difference between MPPVE and previous model-based methods in updating the actor is demonstrated in Figure \ref{fig0}.

In general, our contributions are summarized as follows:
\begin{itemize}
    \item We present a tabular planning policy iteration method that alternates between planning policy evaluation and planning policy improvement, and theoretically prove the convergence of the optimal policy.
    \item We propose a new model-based algorithm called MPPVE for general continuous settings based on the theoretical tabular planning policy iteration, which updates the policy by directly computing multi-step policy gradients via plan value estimation for short plans starting from real environmental states, mitigating the misleading impacts of the compounding error.
    \item We empirically verify that multi-step policy gradients computed by MPPVE are less influenced by model error and more accurate than those computed by previous model-based RL methods.
    \item We show that MPPVE outperforms recent state-of-the-art model-based algorithms in terms of sample efficiency while retaining competitive performance close to the convergence of model-free algorithms on MuJoCo \cite{mujoco} benchmarks.
\end{itemize}

\section{Related Work}

This work is related to dyna-style MBRL and multi-step planning. 

Dyna-style MBRL methods generate some fake transitions with a dynamics model for accelerating value approximation or policy learning. MVE \cite{mve} uses a dynamics model to simulate the short-term horizon and Q-Learning to estimate the long-term value, improving the quality of target values for training. SLBO \cite{slbo} regards the dynamics model as a simulator and directly uses TRPO \cite{trpo} to optimize the policy with whole trajectories sampled in it. Moreover, MBPO \cite{mbpo} builds on SAC \cite{sac}, which is an off-policy RL algorithm, and updates the policy with a mixture of the data from the real environment and imaginary branched rollouts. Some recent dyna-style MBRL methods pay attention to reducing the influences of model error. For instance, M2AC \cite{m2ac} masks the high-uncertainty model-generated data with a masking mechanism. BMPO \cite {bmpo} utilizes both the forward and backward model to separate the compounding error in different directions, which can acquire model data with less error than only using the forward model. This work also adopts the dyna-style MBRL framework and focuses on mitigating the impact of the accumulated model error.

Multi-step planning methods usually utilize the learned dynamics model to make plans. Model Predictive Control (MPC) \cite{mpc} obtains an optimal action sequence by sampling multiple sequences and applying the first action of the sequence to the environment. MB-MF \cite{mb-mf} adopts the random-shooting method as an instantiation of MPC, which samples several action sequences randomly and uniformly in a learned neural model. PETS \cite{pets} uses CEM \cite{cem} instead, which samples actions from a distribution close to previous samples that yielded high rewards to improve the optimization efficiency. The planning can be incorporated into the differentiable neural network architecture to learn the planning policy end-to-end directly \cite{pinet, upn, dan}. Furthermore, MAAC \cite{maac} proposes to estimate the policy gradients by backpropagating through the learned dynamics model using the path-wise derivative estimator during model-based planning. This work also learns the planning policy end-to-end, but backpropagates the policy gradients through multi-step plan value instead, avoiding the explicit reliance on the rolled-out fake states.

The concept of multi-step sequential actions is introduced in other model-based approaches. \cite{msmbrl, m3, ke2018modeling, che2018combining} propose a multi-step dynamics model to directly output the outcome of executing a sequence of actions. Concretely, \cite{msmbrl} aims to use the multi-step model for value-function optimization in the context of actor-critic framework. \cite{m3} combats the compounding error problem by making rollouts in the multi-step model. \cite{ke2018modeling} focuses on building a multi-step model that reasons about the long-term future and uses this for efficient planning and exploration. \cite{che2018combining} leverages the multi-step model to provide the imaginary cumulative rewards as control variates for policy optimization. Unlike these, we involve the concept of multi-step sequential actions in the value function instead of the dynamics model.

Multi-step plan value used in this paper is also presented in GPM \cite{gpm}. GPM can generate actions not only for the current step but also for several future steps via the plan value function, which brings exploration benefits. Our approach differs from GPM in the following two aspects: 1) GPM adopts the model-free paradigm while ours adopts the model-based paradigm; 2) GPM aims to enhance exploration and therefore present a plan generator to output multi-step actions based on the plan value, while we propose model-based planning policy improvement based on the plan value estimation for less influence of compounding error and higher sample efficiency. 

\section{Preliminaries}
A tuple $(\mathcal{S},\mathcal{A},p,r,\gamma)$ is considered to describe Markov Decision Process (MDP), where $\mathcal{S}$ and $\mathcal{A}$ are state and action spaces respectively, probability density function $p:\mathcal{S}\times\mathcal{S}\times\mathcal{A}\to[0,\infty)$ represents the transition distribution over $s_{t+1}\in\mathcal{S}$ conditioned on $s_t\in\mathcal{S}$ and $a_t\in\mathcal{A}$, $r:\mathcal{S}\times\mathcal{A}\times\mathbb{R}\to[0,\infty)$ is the reward distribution conditioned on $s_t\in\mathcal{S}$ and $a_t\in\mathcal{A}$, and $\gamma$ is the discount factor. We will use $\rho^\pi:\mathcal{S}\to [0,\infty)$ to denote on-policy distribution over states induced by dynamics function $p(s_{t+1}|s_t,a_t)$ and policy $\pi(a_t|s_t)$. The goal of RL is to find the optimal policy that maximizes the expectation of cumulative discounted reward: $\mathbb{E}_{\rho^\pi}\left[\sum_{t=0}^\infty \gamma^t r(s_t,a_t)\right]$.

Recent MBRL methods aim to build a model of dynamics function using supervised learning with data $\mathcal{D}_{\mathrm{env}}$ collected via interaction in the real environment. Then fake data $\mathcal{D}_{\mathrm{model}}$ generated via model rollouts will be used for an RL algorithm additionally with real data to improve the sample efficiency.

\section{Method}
In this section, we will first derive tabular planning policy iteration, and verify that the optimal policy can be attained with a convergence guarantee. Then we will present a practical neuron-based algorithm for general continuous environments based on this theory.

\subsection{Derivation of Planning Policy Iteration}
Planning policy iteration is a multi-step extension of policy iteration for optimizing planning policy that alternates between planning policy evaluation and planning policy improvement. With the consideration of theoretical analysis, our derivation will focus on the tabular setting.

\subsubsection{Planning Policy}
With the dynamics function $p$, at any state, the agent can generate a plan consisting of a sequence of actions to perform in the next few steps in turn. We denote the $k$-step planning policy as $\boldsymbol{\pi}^k$, then given state $s_t$, the plan $\boldsymbol{\tau}_t^k=(a_t,a_{t+1},...,a_{t+k-1})$ can be predicted with
\begin{equation}
\label{pp1}
\boldsymbol{\tau}_t^k\sim\boldsymbol{\pi}^k(\cdot|s_t).
\end{equation}
Concretely, for $m\in[0,k-1]$, we have
\begin{equation}
\label{pp2}
\quad a_{t+m} \sim \pi(\cdot|s_{t+m}),
\end{equation}
\begin{equation}
\label{pp3}
s_{t+m+1} \sim p(\cdot|s_{t+m},a_{t+m}),
\end{equation}
\begin{equation}
\label{pp4}
\boldsymbol{\pi}^k(\boldsymbol{\tau}_t^k|s_t)=\prod_{m=0}^{k-1}\pi(a_{t+m}|s_{t+m}).
\end{equation}
The planning policy just gives a temporally consecutive plan of fixed length $k$, not a full plan that plans actions until the termination can be reached. 

\subsubsection{Planning Policy Evaluation}

Given a stochastic policy $\pi\in\Pi$, the $k$-step plan value \cite{gpm} function is defined as
\begin{align}
&Q^\pi(s_t,\boldsymbol{\tau}_t^k)=Q^\pi(s_t,a_t,a_{t+1},...,a_{t+k-1})\nonumber\\
=&\mathbb{E}_{p,r,\pi}\left[\sum_{m=0}^{k-1}\gamma^m r_{t+m}+\gamma^k \sum_{m=0}^{\infty}\gamma^m r_{t+k+m}\right]\\
=&\label{dp}\mathbb{E}_{p,r}\left[\sum_{m=0}^{k-1}\gamma^m r_{t+m}+\gamma^k \mathbb{E}_{\hat{\boldsymbol{\tau}}^k\sim \boldsymbol{\pi}^k}\left[Q^\pi(s_{t+k},\hat{\boldsymbol{\tau}}^k)\right]\right].
\end{align}
The former $k$ items $(r_t,r_{t+1},...,r_{t+k-1})$ are instant rewards respectively for $k$-step actions in the plan $\boldsymbol{\tau}_t^k$ taken at state $s_t$ step by step, while the following items are future rewards for starting making decisions according to $\pi$ from state $s_{t+k}$.

According to the recursive Bellman equation \eqref{dp}, the extended Bellman backup operator $\mathcal{T}^\pi$ can be written as
\begin{equation}
\mathcal{T}^\pi Q(s_t,\boldsymbol{\tau}_t^k)=\mathbb{E}_{p,r}\left[\sum_{m=0}^{k-1}\gamma^m r_{t+m}\right]+\gamma^k\mathbb{E}_{p}\left[ V(s_{t+k})\right],
\end{equation}
where
\begin{equation}
V(s_t) = \mathbb{E}_{\boldsymbol{\tau}_t^k\sim \boldsymbol{\pi}^k}\left[Q(s_{t},\boldsymbol{\tau}_t^k)\right]
\end{equation}
is the state value function. This update equation for $k$-step plan value function is different from multi-step Temporal Difference (TD) Learning \cite{msrl}, though there is naturally multi-step bootstrapping in plan value learning. Specifically, the plan value is updated without bias since the multi-step rewards $(r_t,r_{t+1},...,r_{t+k-1})$ correspond to the input plan, while the multi-step TD Learning for single-step action-value needs importance sampling or Tree-backup \cite{msrl} to correct the bias coming from the difference between target policy and behavior policy.

Starting from any function $Q:\mathcal{S}\times\mathcal{A}^k\to \mathbb{R}$ and applying $\mathcal{T}^\pi$ repeatedly, we can obtain the plan value function of the fixed policy $\pi$.

\begin{lemma}[Planning Policy Evaluation]
\label{lemma1}
Given any initial mapping $Q_0:\mathcal{S}\times\mathcal{A}^k\to \mathbb{R}$ with $|\mathcal{A}|<\infty$, update $Q_i$ to $Q_{i+1}$ with $Q_{i+1}=\mathcal{T}^\pi Q_i$ for all $i\in N$, $\{Q_i\}$ will converge to plan value of policy $\pi$ as $i\to\infty$.
\end{lemma}
\begin{proof}
See Appendix \ref{apdx_lemma1}.
\end{proof}

\subsubsection{Planning Policy Improvement}
After attaining the corresponding plan value, the policy $\pi$ can be updated with
\begin{equation}
\label{policyimprovement}
    \pi_{\mathrm{new}} = \arg\max_{\pi\in\Pi} \sum_{\boldsymbol{\tau}_t^k\in\mathcal{A}^k}\boldsymbol{\pi}^k(\boldsymbol{\tau}_t^k|s_t)Q^{\pi_{\mathrm{old}}}(s_t,\boldsymbol{\tau}_t^k)
\end{equation}
for each state. We will show that $\pi_{\mathrm{new}}$ achieves greater plan value than $\pi_{\mathrm{old}}$ after applying Eq.\eqref{policyimprovement}.

\begin{lemma}[Planning Policy Improvement]
\label{lemma2}
Given any mapping $\pi_{\mathrm{old}}\in\Pi:\mathcal{S}\to\Delta(\mathcal{A})$ with $|\mathcal{A}|<\infty$, update $\pi_\mathrm{old}$ to $\pi_{\mathrm{new}}$ with Eq.\eqref{policyimprovement}, then $Q^{\pi_{\mathrm{new}}}(s_t,\boldsymbol{\tau}_t^k)\geq Q^{\pi_{\mathrm{old}}}(s_t,\boldsymbol{\tau}_t^k)$, $\forall s_t\in\mathcal{S}$, $\boldsymbol{\tau}_t^k\in\mathcal{A}^k$.
\end{lemma}
\begin{proof}
See Appendix \ref{apdx_lemma2}.
\end{proof}

\subsubsection{Planning Policy Iteration}
The whole planning policy iteration process alternates between planning policy evaluation and planning policy improvement until the returned sequence of policy $\{\pi_k\}$ converges to the optimal policy $\pi_*$ whose plan value of any state-plan pair is the greatest among all $\pi\in\Pi$.

\begin{thm}[Planning Policy Iteration]
\label{thm1}
Given any initial mapping $\pi_0\in\Pi:\mathcal{S}\to\Delta(\mathcal{A})$ with $|\mathcal{A}|<\infty$, compute corresponding $Q^{\pi_i}$ in planning policy evaluation step and update $\pi_i$ to $\pi_{i+1}$ in planning policy improvement step for all $i\in N$, $\{\pi_i\}$ will converge to the optimal policy $\pi_*$ that $Q^{\pi_*}(s_t,\boldsymbol{\tau}_t^k)\geq Q^{\pi}(s_t,\boldsymbol{\tau}_t^k)$, $\forall s_t\in\mathcal{S}$, $\boldsymbol{\tau}_t^k\in\mathcal{A}^k$, and $\pi\in\Pi$.
\end{thm}
\begin{proof}
See Appendix \ref{apdx_thm1}.
\end{proof}

We remark that our planning policy iteration can be extended to soft form with maximum entropy, referring to SAC \cite{sac}, see Appendix \ref{softppi} for more details.

\subsection{Model-based Planning Policy Learning with Multi-step Plan Value Estimation}
The tabular planning policy iteration cannot be directly applied to scenarios with inaccessible dynamics functions and continuous environmental state-action space. Therefore, we propose a neuron-based algorithm based on planning policy iteration for general application, called \textbf{M}odel-based Planning \textbf{P}olicy Learning with Multi-step \textbf{P}lan \textbf{V}alue \textbf{E}stimation (\textbf{MPPVE}). We only introduce MPPVE with vanilla plan value in this section, while the soft extension of MPPVE with maximum entropy is given in Appendix \ref{softMPPVE}.

As shown in Figure \ref{fig0(c)}, our MPPVE adopts the framework of model-based actor-critic, which is divided into dynamics model $p_\theta$, actor $\pi_\phi$ and critic $Q_\psi$, where $\theta$, $\phi$ and $\psi$ are neural parameters. The algorithm will be described in three parts: 1) model learning; 2) multi-step plan value estimation; 3) model-based planning policy improvement.

\subsubsection{Model Learning} Like MBPO \cite{mbpo}, our dynamics model is an ensemble neural network that takes state-action pair as input and outputs Gaussian distribution of the next state and reward. That is, $p_\theta(s_{t+1},r_t|s_t,a_t)=\mathcal{N}(\mu_\theta(s_t,a_t),\Sigma_\theta(s_t,a_t))$. The dynamics model is trained to maximize the expected likelihood:
\begin{equation}
\label{pobject}
J_p(\theta)=\mathbb{E}_{(s_t,a_t,r_t,s_{t+1})\sim \mathcal{D}_{\mathrm{env}}}\left[\log p_\theta(s_{t+1},r_t|s_t,a_t)\right].
\end{equation}

\begin{algorithm}[t!]
\caption{MPPVE}
\label{MPPVE}
\textbf{Input}: Initial neural parameters $\theta$, $\phi$, $\psi$, $\psi^-$, plan length $k$, environment buffer $\mathcal{D}_{\mathrm{env}}$, model buffer $\mathcal{D}_{\mathrm{model}}$, start size $U$, batch size $B$, and learning rate $\lambda_Q$, $\lambda_\pi$.

\begin{algorithmic}[1] 
\STATE Explore in the environment for $U$ steps and add data to $\mathcal{D}_{\mathrm{env}}$
\FOR{$N$ epochs}
\STATE Train model $p_\theta$ on $\mathcal{D}_{\mathrm{env}}$ by maximizing Eq.\eqref{pobject}
\FOR{$E$ steps}
\STATE Sample action to perform in the environment according to $\pi_\phi$; add the environmental transition to $\mathcal{D}_{\mathrm{env}}$
\FOR{$M$ model rollouts}
\STATE Sample $s_t$ from $\mathcal{D}_{\mathrm{env}}$ to make model rollout using policy $\pi_\phi$; add generated samples to $\mathcal{D}_{\mathrm{model}}$
\ENDFOR
\FOR{$G$ critic updates}
\STATE Sample $B$ $k$-step trajectories from $\mathcal{D}_{\mathrm{env}}\cup\mathcal{D}_{\mathrm{model}}$ to update critic $Q_\psi$ via $\psi\leftarrow\psi-\lambda_Q\hat{\nabla}_\psi J_Q(\psi)$ by Eq.\eqref{Qgrad}
\STATE Update target critic via $\psi^- \leftarrow \tau\psi+(1-\tau)\psi^-$
\ENDFOR
\STATE Sample $B$ states from $\mathcal{D}_{\mathrm{env}}$ to update policy $\pi_\phi$ via $\phi \leftarrow \phi- \lambda_\pi\hat{\nabla}_\phi J_{\boldsymbol{\pi}^k}(\phi)$ by Eq.\eqref{mspg}
\ENDFOR
\ENDFOR
\end{algorithmic}
\end{algorithm}

\begin{figure*}[t!]
\centering
\includegraphics[width=0.94\linewidth]{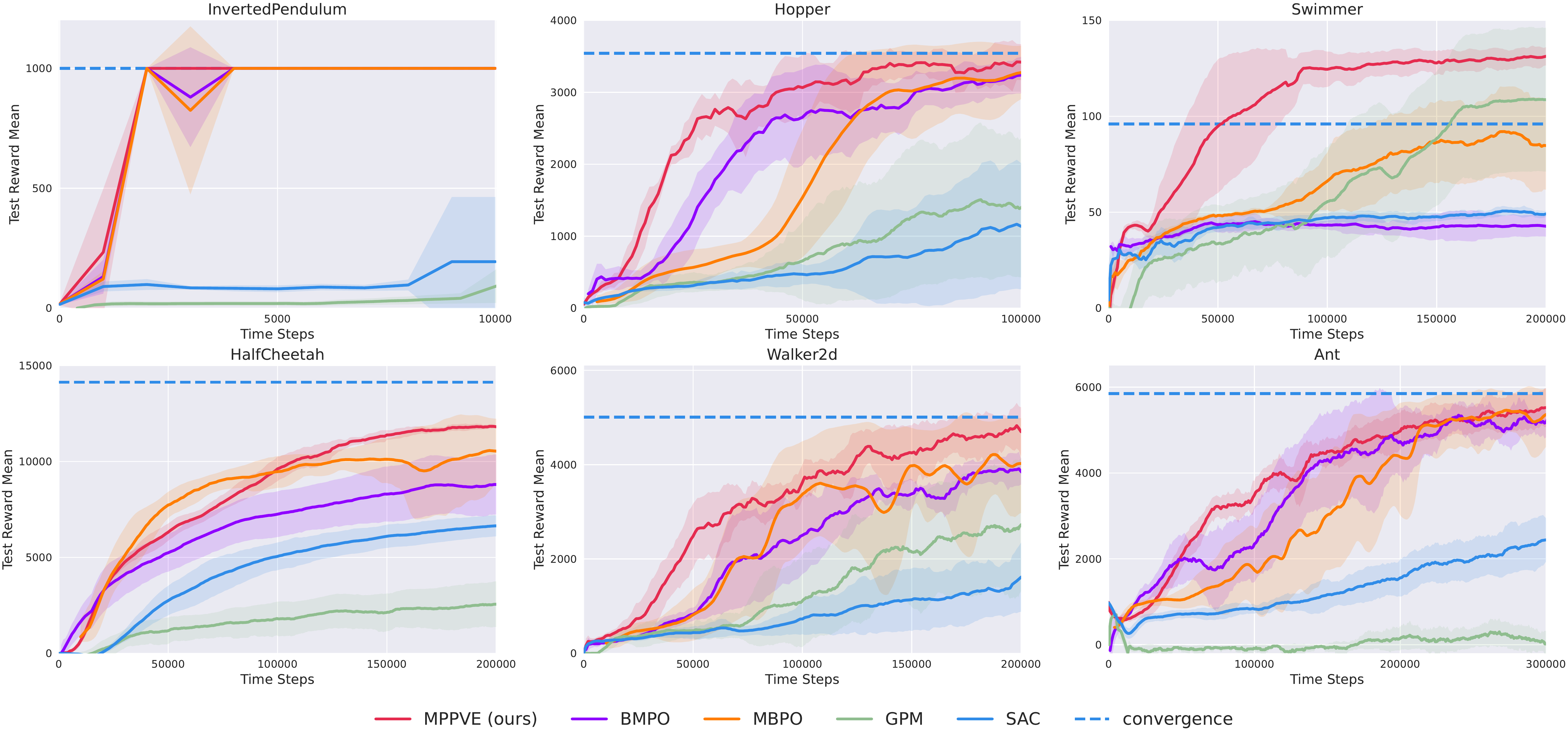}
\caption{Learning curves of our MPPVE (red) and other four baselines on MuJoCo continuous control tasks. The blue dashed lines indicate the asymptotic performance of SAC in these tasks for reference. The solid lines indicate the mean, and shaded areas indicate the standard error over eight different random seeds. Each evaluation, taken every 1,000 environmental steps, calculates the average return over ten episodes.}
\label{fig3}
\end{figure*}

\subsubsection{Multi-step Plan Value Estimation} In order to approximate $k$-step plan value function with continuous inputs, we use a deep Q-network: $Q_\psi(s_t,\boldsymbol{\tau}^k_t)$ with parameters $\psi$. Plan value is estimated by minimizing the expected multi-step TD error:
\begin{equation}
\label{Qobject}
J_Q(\psi) = \mathbb{E}_{(s_t,\boldsymbol{\tau}^k_t,\boldsymbol{r}^k_t,s_{t+k})\sim \mathcal{D}_{\mathrm{env}}\cup \mathcal{D}_{\mathrm{model}}}\left[l_{\mathrm{TD}}(s_t,\boldsymbol{\tau}^k_t,\boldsymbol{r}^k_t,s_{t+k})\right],
\end{equation}
with
\begin{equation}
\boldsymbol{r}^k_t = (r_t,r_{t+1},...,r_{t+k-1}),
\end{equation}
\begin{equation}
l_{\mathrm{TD}}(s_t,\boldsymbol{\tau}^k_t,\boldsymbol{r}^k_t,s_{t+k})=\frac{1}{2}\left(Q_\psi(s_t,\boldsymbol{\tau}^k_t)
-y_{\psi^-}(\boldsymbol{r}^k_t,s_{t+k})\right)^2,
\end{equation}
\begin{equation}
y_{\psi^-}(\boldsymbol{r}^k_t,s_{t+k})=\sum_{m=0}^{k-1}\gamma^m r_{t+m}+\gamma^k \mathbb{E}_{\hat{\boldsymbol{\tau}}^k}\left[Q_{\psi^-}(s_{t+k},\hat{\boldsymbol{\tau}}^k)\right],
\end{equation}
where $\psi^-$ is the parameters of the target network, which is used for stabilizing training \cite{dqn}. The gradient of Eq.\eqref{Qobject} can be estimated without bias by
\begin{equation}
\label{Qgrad}
\hat{\nabla}_\psi J_Q(\psi)=
\left(Q_\psi(s_t,\boldsymbol{\tau}^k_t)-\hat{y}_{\psi^-}(\boldsymbol{r}^k_t,s_{t+k})\right) \nabla_\psi Q_\psi(s_t,\boldsymbol{\tau}^k_t),
\end{equation}
with
\begin{equation}
\hat{y}_{\psi^-}(\boldsymbol{r}^k_t,s_{t+k})=\sum_{m=0}^{k-1}\gamma^m r_{t+m}+\gamma^k Q_{\psi^-}(s_{t+k},\boldsymbol{\tau}_{t+k}^k),
\end{equation}
where $(s_t,\boldsymbol{\tau}^k_t,\boldsymbol{r}^k_t,s_{t+k})$ is sampled from the replay buffer and $\boldsymbol{\tau}_{t+k}^k$ is sampled according to current planning policy.

\subsubsection{Model-based Planning Policy Improvement}
As shown in Eq.\eqref{pp1}-Eq.\eqref{pp4}, a $k$-step model-based planning policy $\boldsymbol{\pi}^k_{\phi,\theta}$ is composed of dynamics model $p_\theta$ and policy $\pi_\phi$. Since the plan value function is represented by a differentiable neural network, we can train model-based planning policy by minimizing
\begin{equation}
    J_{\boldsymbol{\pi}^k}(\phi)=\mathbb{E}_{s_t\sim\mathcal{D}_{\mathrm{env}}}\left[\mathbb{E}_{\boldsymbol{\tau}^k_t\sim\boldsymbol{\pi}^k_{\phi,\theta}(\cdot|s_t)}
    \left[-Q_\psi(s_t,\boldsymbol{\tau}^k_t)\right]\right],
\end{equation}
whose gradient can be approximated by
\begin{equation}
\label{mspg}
\hat{\nabla}_\phi J_{\boldsymbol{\pi}^k}(\phi)=\left(-\nabla_{\boldsymbol{\tau}^k_t}Q_\psi(s_t,\boldsymbol{\tau}^k_t)\right) \nabla_\phi \boldsymbol{f}_{\phi,\theta}(s_t,\boldsymbol{\eta^k_t}),
\end{equation}
where $\boldsymbol{f}_{\phi,\theta}(s_t,\boldsymbol{\eta^k_t})=\boldsymbol{\tau}^k_t$ is the reparameterized neural network transformation of planning policy with $\boldsymbol{\eta}_t^k=(\eta_1,\eta_2,...,\eta_k)$, which is a $k$-step noise vector sampled from Gaussian distribution.

During the model-based planning policy improvement step, the actor makes $k$-step plans starting from $s_t$ using policy $\pi_\phi$ and model $p_\theta$. Then, the corresponding state-plan pairs are fed into the neural plan value function to directly compute the gradient \eqref{mspg}. Therefore, the policy gradients within these $k$ steps are mainly affected by the bias of plan value estimation at $s_t$, avoiding the influence of action-value estimation with compounding error at the following $k-1$ states in previous model-based methods.

~\\
The complete algorithm is described in Algorithm \ref{MPPVE}, and the hyper-parameter settings are given in Appendix \ref{hyper}. We also utilize model rollouts to generate fake transitions, as proposed by MBPO \cite{mbpo}. The primary difference from MBPO is that only the critic is trained with additional model data for supporting larger UTD in our algorithm, while the actor is trained merely with real-world data at a lower update frequency than the critic. There are two reasons for the modification: 1) plan value is more difficult to estimate than action-value, then the critic can guide the actor only after learning with sufficient and diverse samples; 2) model data is not required for the actor since a $k$-step model rollout exists naturally in our model-based planning policy. 

\section{Experiments}
In this section, we focus on three primary questions: 1) How well does our method perform on benchmark tasks of reinforcement learning, in contrast to a broader range of state-of-the-art model-free and model-based RL methods? 2) Does our proposed model-based planning policy improvement based on plan value estimation provide more accurate policy gradients than previous model-based RL methods? 3) How the plan length $k$ affects our method?

\begin{figure*}[pt!]
    \centering
    \subfigure[]{
        \centering
        \includegraphics[width=0.31\linewidth]{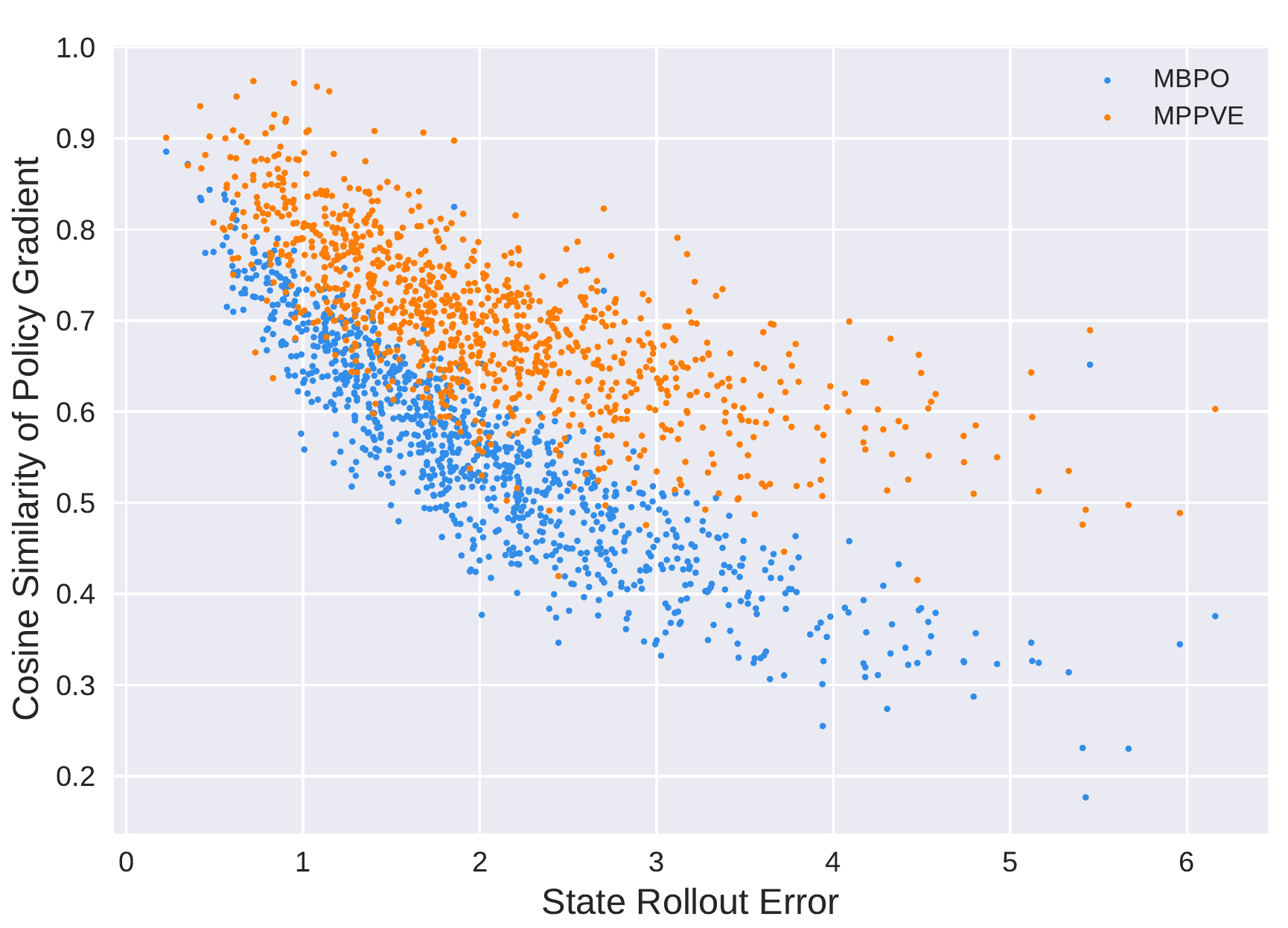}
        \label{fig4(a)}
    }%
    \subfigure[]{
        \centering
        \includegraphics[width=0.31\linewidth]{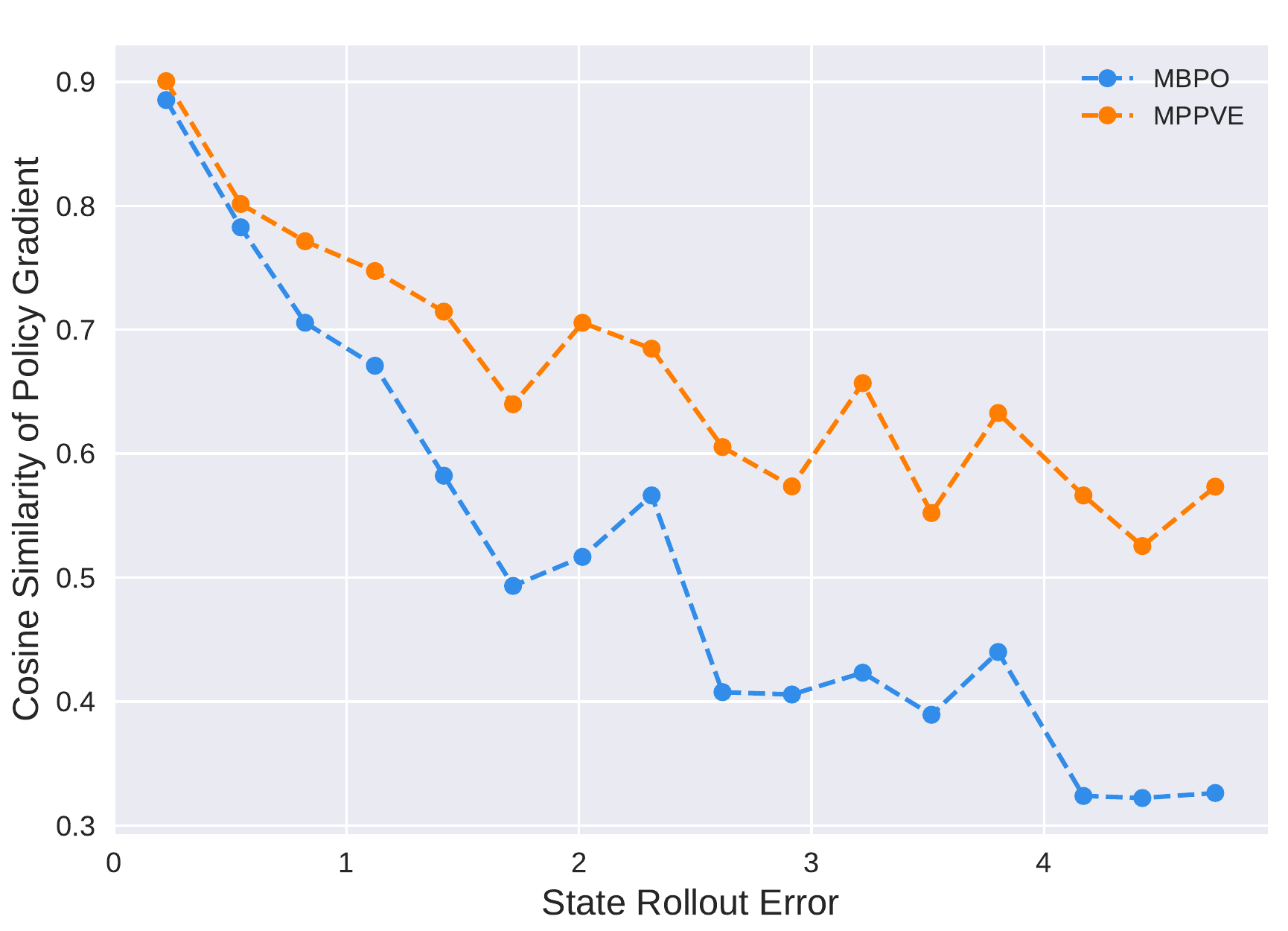}
        \label{fig4(b)}
    }%
    \subfigure[]{
        \centering
        \includegraphics[width=0.31\linewidth]{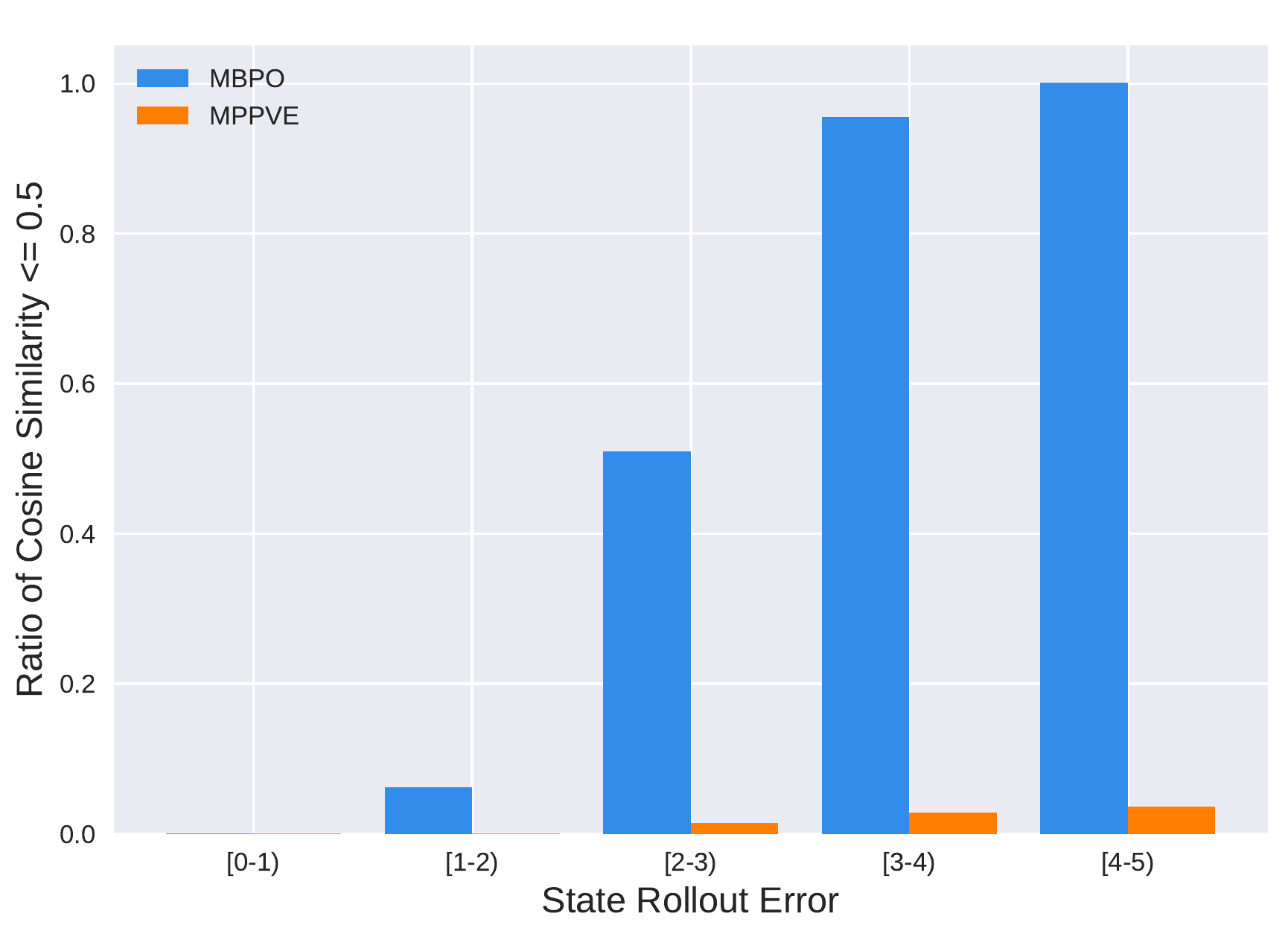}
        \label{fig4(c)}
    }%
    \centering
    \caption{We compare the influence of model error on the directions of $k$-step ($k=3$) policy gradients given by MPPVE and MBPO on the HalfCheetah task. We make real branched rollouts in the real world and utilize the learned model to generate the fake, starting from some real states. (a) For each fake rollout and its corresponding real rollout, we show their deviation along with the normalized cosine similarity between their $k$-step policy gradients. (b) We select some samples and plot further, where each group of orange and blue points on the same X-coordinate corresponds to the same starting real state. (c) We count the ratio of points with severely inaccurate $k$-step policy gradients for each interval of state rollout error.}
    \label{fig4}
\end{figure*}

\subsection{Comparison in RL Benchmarks}
We evaluate our method MPPVE on 6 MuJoCo continuous control tasks \cite{mujoco}, including InvertedPendulum, Hopper, Swimmer, HalfCheetah, Walker2d, and Ant. Two model-free methods and two model-based methods are selected as our baselines. For model-free methods, we compare MPPVE with SAC \cite{sac}, the state-of-the-art model-free algorithm, and GPM \cite{gpm}, which also proposes the concept of plan value and features an inherent mechanism for generating temporally coordinated plans for exploration. For model-based methods, we compare MPPVE with MBPO \cite{mbpo}, the most representative model-based method so far, and BMPO \cite{bmpo}, which proposes a bidirectional dynamics model to generate model data with less compounding error than MBPO.

Figure \ref{fig3} shows the learning curves of all approaches, along with the asymptotic performance of SAC. MPPVE achieves great performance after fewer environmental samples than baselines. Take Hopper as an example, MPPVE has achieved 80\% performance (about 2700) after 30k steps, while the other two model-based methods, BMPO and MBPO, need about 45k and 60k steps respectively, and model-free methods, both SAC and GPM, can achieve only about 1000 after 100k steps. MPPVE performs 1.5x faster than BMPO, 2x faster than MBPO, and dominates SAC and GPM, in terms of learning speed on the Hopper task. After training, MPPVE can achieve a final performance close to the asymptotic performance of SAC on all these six MuJoCo benchmarks. These results show that MPPVE has both high sample efficiency and competitive performance. 

\begin{figure*}[t!]
    \centering
    \subfigure[Performance]{
        \centering
        \includegraphics[width=0.24\linewidth]{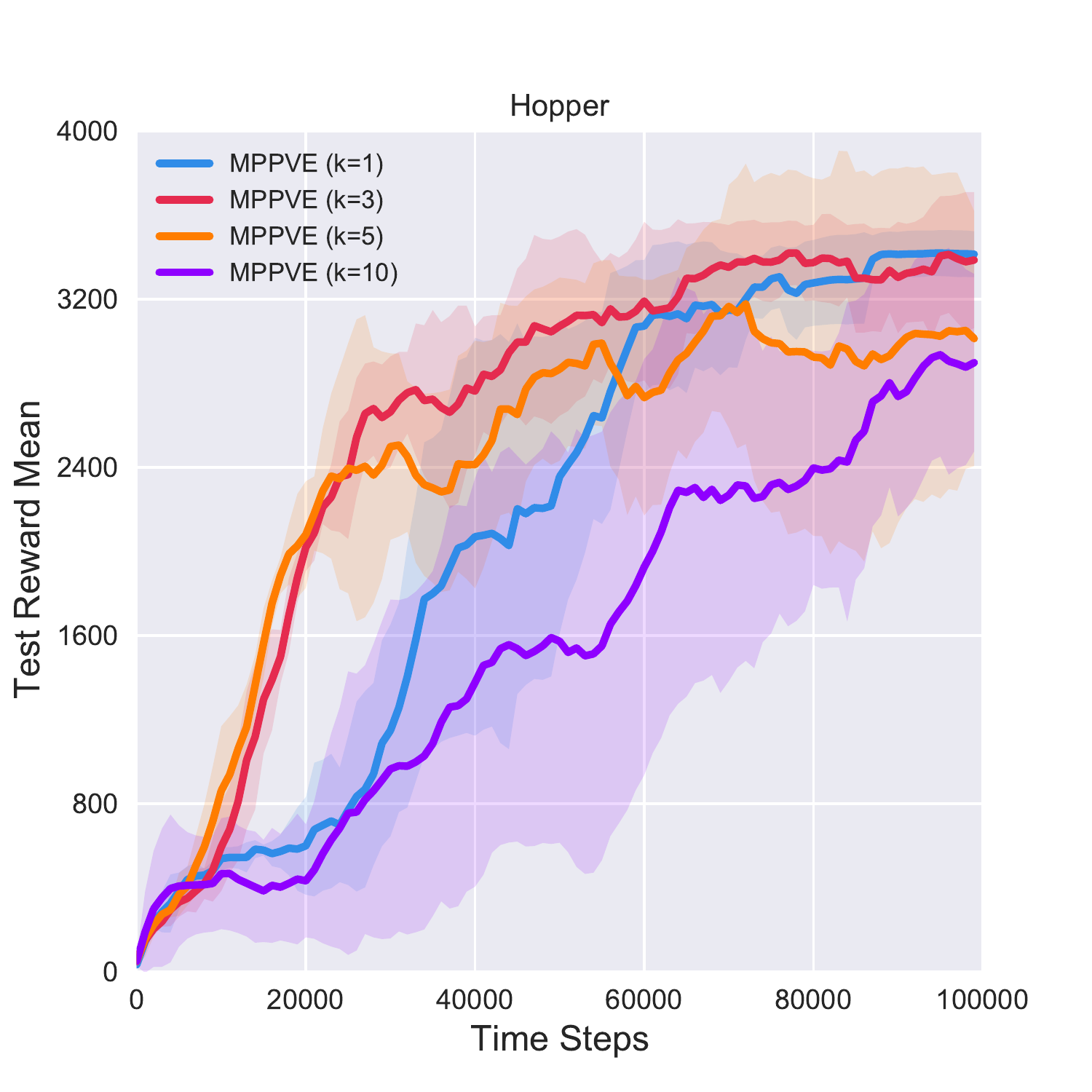}
        \label{fig5(a)}
    }%
    \subfigure[Policy evaluation]{
        \centering
        \includegraphics[width=0.24\linewidth]{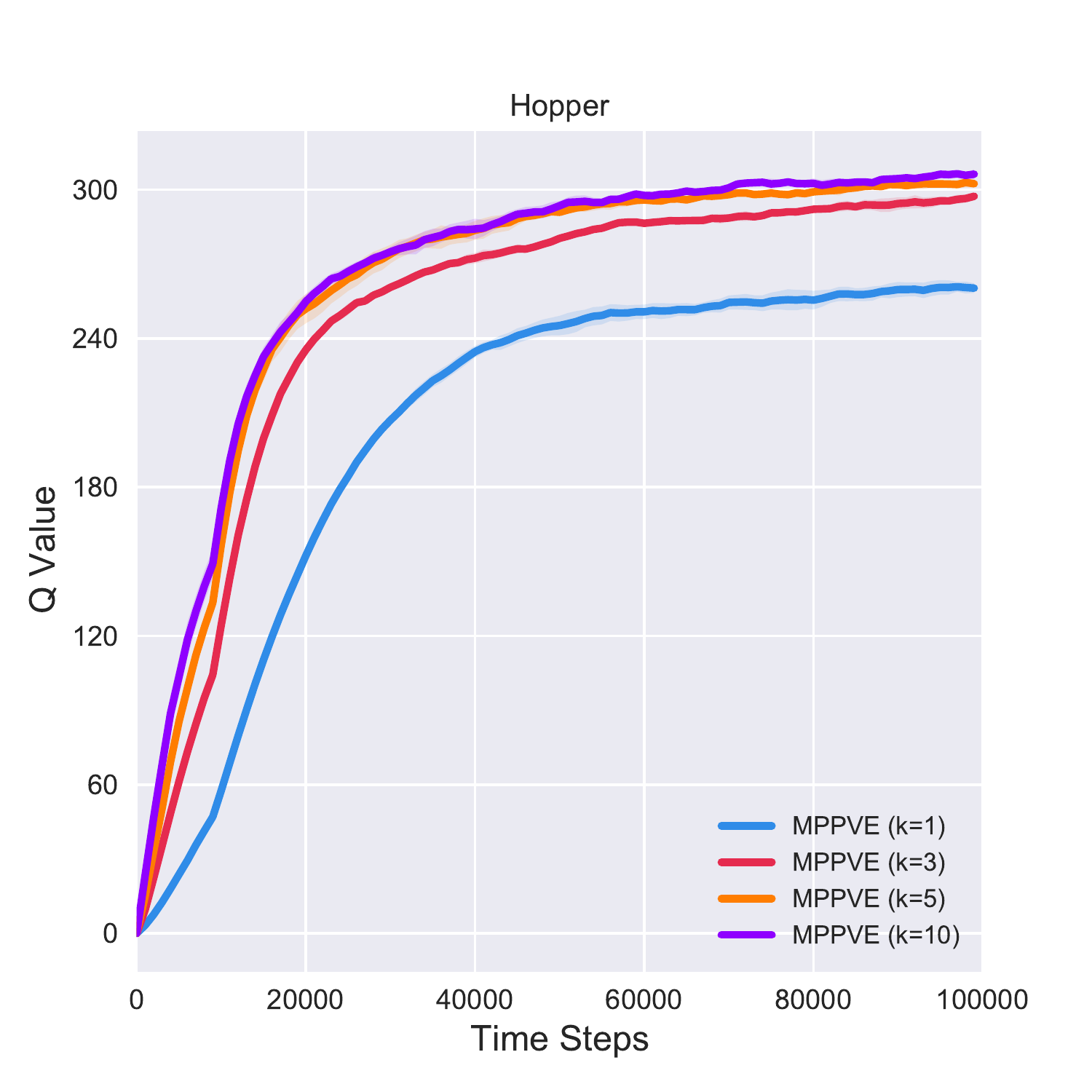}
        \label{fig5(b)}
    }%
    \subfigure[Average of value bias]{
        \centering
        \includegraphics[width=0.24\linewidth]{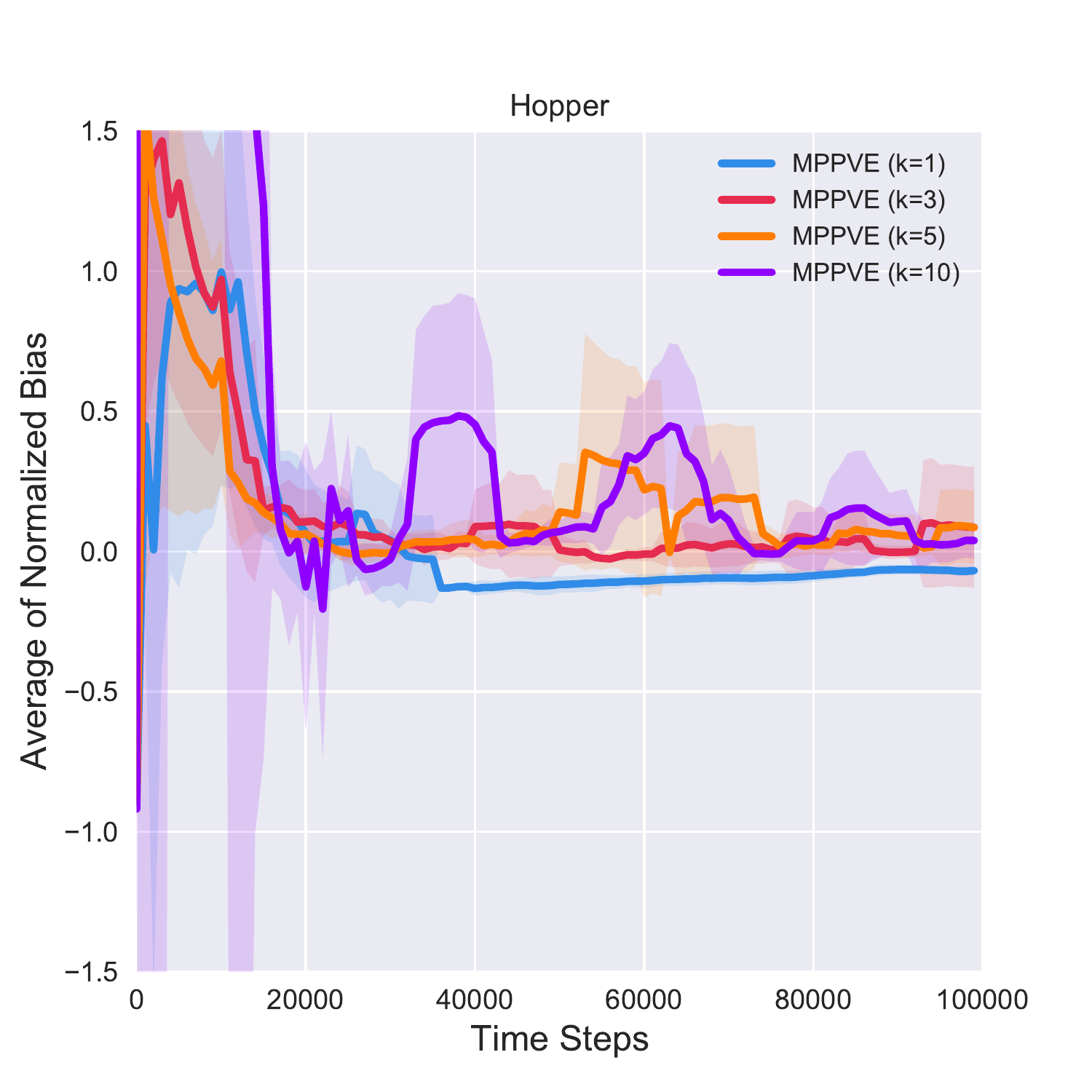}
        \label{fig5(c)}
    }%
    \subfigure[Std of value bias]{
        \centering
        \includegraphics[width=0.24\linewidth]{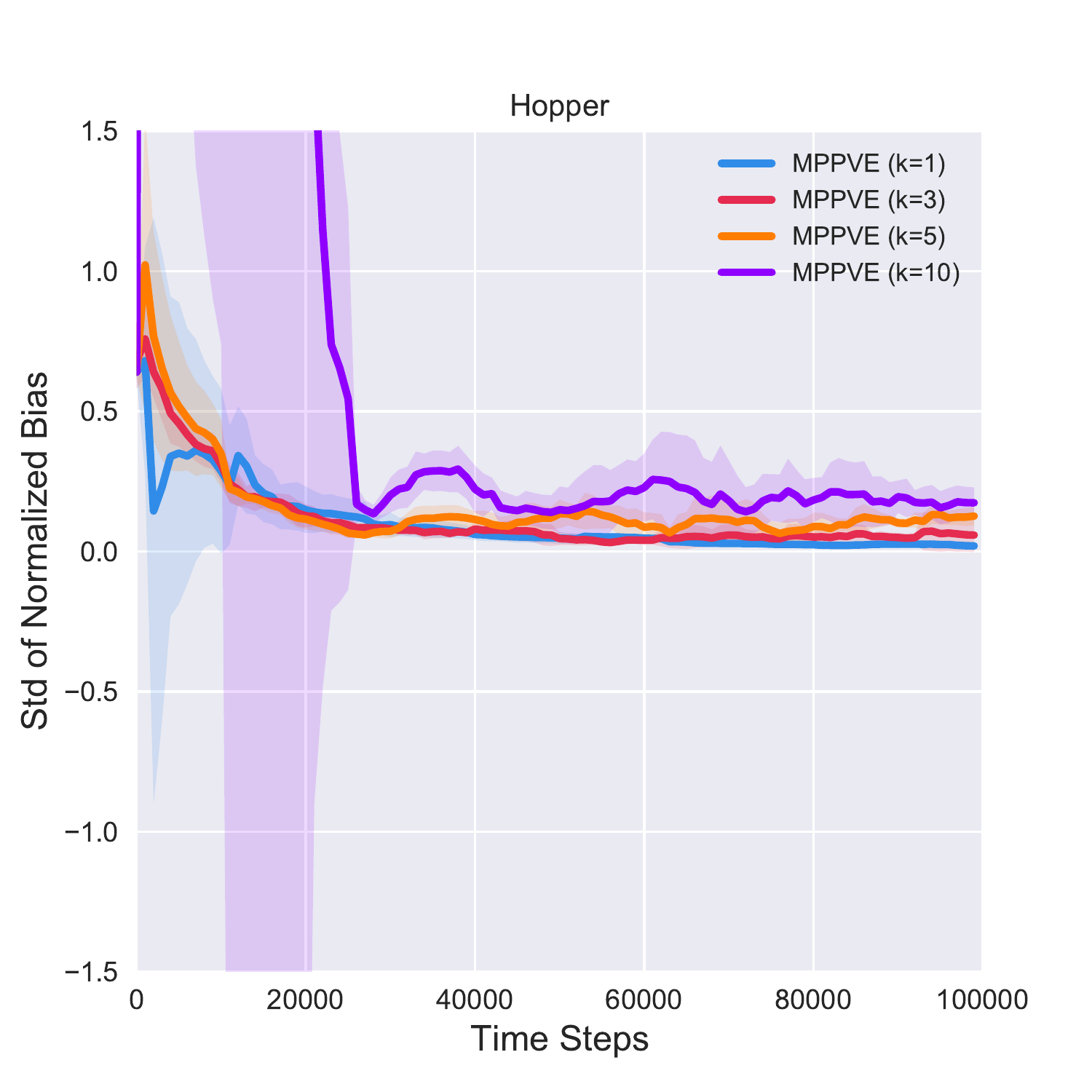}
        \label{fig5(d)}
    }%
    \centering
    \caption{Study of plan length $k$ on Hopper-v3 task, through comparison of MPPVE with $k=1$, $k=3$, $k=5$ and $k=10$ respectively. (a) Evaluation of episodic reward during the learning process. (b) Learning plan value estimation for the same fixed policy. (c) Average of the bias between neural value estimation and true value evaluation over state-action(s) space. (d) Standard error of the bias between neural value estimation and true value evaluation over state-action(s) space.}
    \label{fig5}
\end{figure*}

\subsection{Study on Policy Gradients}
We conduct a study to verify that multi-step policy gradients computed by our MPPVE are less influenced by model error and more accurate than those computed by previous model-based RL methods, as we claimed. Without loss of rigor, we only choose MBPO for comparison since many other model-based methods follow the same way of computing the policy gradients as MBPO.

In order to exclude irrelevant factors as far as possible, we fix the policy and the learned dynamics model after enough environmental samples, then learn the multi-step plan value function and action-value function, respectively, until they are both converged. Next, we measure the influence of model error on the directions of policy gradients computed by MPPVE and MBPO, respectively. Specifically, we sample some real environmental states from the replay buffer and start from them to make $k$-step real rollouts with perfect oracle dynamics and generate $k$-step imaginary rollouts with the learned model. Both MPPVE and MBPO can compute $k$-step policy gradients on the generated fake rollouts and their corresponding real rollouts. MPPVE utilizes $k$-step plan value estimation on the first states of the rollouts to compute $k$-step policy gradients directly, while MBPO needs to compute the single-step policy gradients with action-value estimation on all states of the rollouts and averages over the $k$ steps. For each fake rollout and its corresponding real rollout, we can measure their deviation along with the normalized cosine similarity between their $k$-step policy gradients given by MBPO or MPPVE.

Figure \ref{fig4(a)} shows the influence of model error on the directions of $k$-step ($k=3$) policy gradients on the HalfCheetah task, where orange stands for MPPVE while blue stands for MBPO. Each point corresponds to one starting real state for making $k$-step rollouts, and its X-coordinate is the error of the fake rollout from the real rollout, while its Y-coordinate is the normalized cosine similarity between $k$-step policy gradients computed on the real rollout and the fake rollout. On the one hand, it can be observed that the point with less state rollout error tends to have a smaller normalized cosine similarity of policy gradients, both for MPPVE and MBPO. The tendency reveals the influence of model error on the policy gradients. On the other hand, the figure demonstrates that the orange points are almost overall above the blue points, which means that the $k$-step policy gradients computed by our MPPVE are less influenced by model error.

For the sake of clear comparison, we select only about a dozen real environmental states and plot further. Figure \ref{fig4(b)} shows the result, where each group of orange and blue points on the same X-coordinate corresponds to the same starting real state for rollout. We connect the points of the same color to form a line. The orange line is always above the blue line, indicating the advantage of MPPVE in computing policy gradients.

Furthermore, we count the ratio of points with normalized cosine similarity between $k$-step policy gradients computed on real and fake rollouts smaller than 0.5 for each interval of state rollout error, as shown in Figure \ref{fig4(c)}. A normalized cosine similarity smaller than 0.5 means that the direction of the $k$-step policy gradient deviates by more than 90 degrees and is severely inaccurate. As for MBPO, the ratio of severely biased policy gradients increases as the state rollout error increases and reaches approximately 1 when the state rollout error is great enough. By contrast, our MPPVE only provides severely inaccurate policy gradients with a tiny ratio, even under a great state rollout error.

In summary, we can conclude that by directly computing multi-step policy gradients via plan value estimation, the key idea of our MPPVE provides policy gradients more accurately and effectively than MBPO, whose computation of policy gradients is also adopted by other previous state-of-the-art model-based RL methods.

\subsection{Study on Plan Length}

The previous section empirically shows the advantage of the model-based planning policy improvement. Furthermore, a natural question is what plan length $k$ is appropriate. Intuitively, a larger length allows the policy to be updated more times, which may improve the sample efficiency. Nevertheless, it also makes learning the plan value function more difficult. We next make ablations to our method to better understand the effect of plan length $k$.

Figure \ref{fig5} presents an empirical study of this hyper-parameter. The performance increases first and then decreases as plan length $k$ increases. Specifically, $k=3$ and $k=5$ learn fastest at the beginning of training followed by $k=1$ and $k=10$, while $k=1$ and $k=3$ are more stable than $k=5$ and $k=10$ in the subsequent training phase and also have better performance. Figure \ref{fig5(b)}, \ref{fig5(c)} and \ref{fig5(d)} explain why they perform differently. We plot their policy evaluation curves for the same fixed policy in Figure \ref{fig5(b)}. It indicates that a larger $k$ can make learning the plan value function faster. We then quantitatively analyze their estimation quality in Figure \ref{fig5(c)} and \ref{fig5(d)}. Specifically, we define the normalized bias of the $k$-step plan value estimation $Q_{\psi}(s_t, \boldsymbol{\tau}_t^k)$ to be $(Q_{\psi}(s_t, \boldsymbol{\tau}_t^k) - Q^{\pi}(s_t, \boldsymbol{\tau}_t^k))/|E_{s\sim \rho^\pi}[E_{\boldsymbol{\tau}^k \sim \boldsymbol{\pi}^k(\cdot|s)} [Q^{\pi} (s, \boldsymbol{\tau}^k)]]|$, where actual plan value $ Q^{\pi}(s_t, \boldsymbol{\tau}_t^k)$ is obtained by Monte Carlo sampling in the real environment. Strikingly, compared to the other two settings, $k=5$ and $k=10$ have quite high normalized average and std of bias during training, indicating the value function with too large plan length is hard to be fitted and then causes fluctuations in policy learning. In conclusion, $k=3$ achieves a trade-off between stability and learning speed of plan value estimation, so it performs best.

\section{Conclusion}
In this work, we propose a novel model-based reinforcement method, namely \textbf{M}odel-based Planning \textbf{P}olicy Learning with Multi-step \textbf{P}lan \textbf{V}alue \textbf{E}stimation (MPPVE). The algorithm is from the tabular planning policy iteration, whose theoretical derivation shows that any initial policy can converge to the optimal policy when applied in tabular settings. For general continuous settings, we empirically show that directly computing multi-step policy gradients via plan value estimation when the policy plans start from real states, the key idea of MPPVE, is less influenced by model error and provides more accurate policy gradients than previous model-based RL methods. Experimental results demonstrate that MPPVE achieves better sample efficiency than previous state-of-the-art model-free and model-based methods while retaining competitive performance on several continuous control tasks. In the future, we will explore the scalability of MPPVE and study how to estimate plan value stably when plan length is large to improve sample efficiency further.


\bibliography{aaai23}

\begin{thebibliography}{10}

\bibitem{msmbrl}
K.~Asadi, E.~Cater, D.~Misra, and M.~L. Littman.
\newblock Towards a simple approach to multi-step model-based reinforcement
  learning.
\newblock {\em CoRR}, abs/1811.00128, 2018.

\bibitem{m3}
K.~Asadi, D.~Misra, S.~Kim, and M.~L. Littman.
\newblock Combating the compounding-error problem with a multi-step model.
\newblock {\em CoRR}, abs/1905.13320, 2019.

\bibitem{msrl}
K.~D. Asis, J.~F. Hernandez{-}Garcia, G.~Z. Holland, and R.~S. Sutton.
\newblock Multi-step reinforcement learning: {A} unifying algorithm.
\newblock In {\em Proceedings of the 32nd {AAAI} Conference on Artificial
  Intelligence (AAAI'18)}, New Orleans, Louisiana, 2018.

\bibitem{cem}
Z.~I. Botev, D.~P. Kroese, R.~Y. Rubinstein, and P.~L’Ecuyer.
\newblock Chapter 3 - the cross-entropy method for optimization.
\newblock In {\em Handbook of Statistics}, volume~31 of {\em Handbook of
  Statistics}, pages 35--59. Elsevier, 2013.

\bibitem{mpc}
E.~Camacho and C.~Alba.
\newblock {\em Model Predictive Control}.
\newblock Advanced Textbooks in Control and Signal Processing. Springer London,
  2013.

\bibitem{che2018combining}
T.~Che, Y.~Lu, G.~Tucker, S.~Bhupatiraju, S.~Gu, S.~Levine, and Y.~Bengio.
\newblock Combining model-based and model-free {RL} via multi-step control
  variates.
\newblock \url{https://openreview.net}, 2018.

\bibitem{chen2022arxiv}
X.-H. Chen, Y.~Yu, Z.-M. Zhu, Z.~Yu, Z.~Chen, C.~Wang, Y.~Wu, H.~Wu, R.-J. Qin,
  R.~Ding, and F.~Huang.
\newblock Adversarial counterfactual environment model learning.
\newblock {\em arXiv preprint arXiv:2206.04890}, 2022.

\bibitem{pets}
K.~Chua, R.~Calandra, R.~McAllister, and S.~Levine.
\newblock Deep reinforcement learning in a handful of trials using
  probabilistic dynamics models.
\newblock In {\em Advances in Neural Information Processing Systems 31
  (NeurIPS'18)}, Montr{\'{e}}al, Canada, 2018.

\bibitem{maac}
I.~Clavera, Y.~Fu, and P.~Abbeel.
\newblock Model-augmented actor-critic: Backpropagating through paths.
\newblock In {\em 8th International Conference on Learning Representations
  (ICLR'20)}, Addis Ababa, Ethiopia, 2020.

\bibitem{mve}
V.~Feinberg, A.~Wan, I.~Stoica, M.~I. Jordan, J.~E. Gonzalez, and S.~Levine.
\newblock Model-based value estimation for efficient model-free reinforcement
  learning.
\newblock {\em CoRR}, abs/1803.00101, 2018.

\bibitem{sac}
T.~Haarnoja, A.~Zhou, P.~Abbeel, and S.~Levine.
\newblock Soft actor-critic: Off-policy maximum entropy deep reinforcement
  learning with a stochastic actor.
\newblock In {\em Proceedings of the 35th International Conference on Machine
  Learning (ICML'18)}, Stockholm, Sweden, 2018.

\bibitem{mbpo}
M.~Janner, J.~Fu, M.~Zhang, and S.~Levine.
\newblock When to trust your model: Model-based policy optimization.
\newblock In {\em Advances in Neural Information Processing Systems 32
  (NeurIPS'19)}, Vancouver, Canada, 2019.

\bibitem{dan}
P.~Karkus, X.~Ma, D.~Hsu, L.~P. Kaelbling, W.~S. Lee, and
  T.~Lozano{-}P{\'{e}}rez.
\newblock Differentiable algorithm networks for composable robot learning.
\newblock In {\em Robotics: Science and Systems XV (RSS'19)}, Freiburg im
  Breisgau, Germany, 2019.

\bibitem{ke2018modeling}
N.~R. Ke, A.~Singh, A.~Touati, A.~Goyal, Y.~Bengio, D.~Parikh, and D.~Batra.
\newblock Modeling the long term future in model-based reinforcement learning.
\newblock In {\em 7th International Conference on Learning Representations
  (ICLR'19)}, New Orleans, LA, 2019.

\bibitem{bmpo}
H.~Lai, J.~Shen, W.~Zhang, and Y.~Yu.
\newblock Bidirectional model-based policy optimization.
\newblock In {\em Proceedings of the 37th International Conference on Machine
  Learning (ICML'20)}, Virtual Conference, 2020.

\bibitem{slbo}
Y.~Luo, H.~Xu, Y.~Li, Y.~Tian, T.~Darrell, and T.~Ma.
\newblock Algorithmic framework for model-based deep reinforcement learning
  with theoretical guarantees.
\newblock In {\em 7th International Conference on Learning Representations
  (ICLR'19)}, New Orleans, LA, 2019.

\bibitem{dqn}
V.~Mnih, K.~Kavukcuoglu, D.~Silver, A.~A. Rusu, J.~Veness, M.~G. Bellemare,
  A.~Graves, M.~A. Riedmiller, A.~Fidjeland, G.~Ostrovski, S.~Petersen,
  C.~Beattie, A.~Sadik, I.~Antonoglou, H.~King, D.~Kumaran, D.~Wierstra,
  S.~Legg, and D.~Hassabis.
\newblock Human-level control through deep reinforcement learning.
\newblock {\em Nature}, 518(7540):529--533, 2015.

\bibitem{mb-mf}
A.~Nagabandi, G.~Kahn, R.~S. Fearing, and S.~Levine.
\newblock Neural network dynamics for model-based deep reinforcement learning
  with model-free fine-tuning.
\newblock In {\em 2018 {IEEE} International Conference on Robotics and
  Automation (ICRA'18)}, Brisbane, Australia, 2018.

\bibitem{pinet}
M.~Okada, L.~Rigazio, and T.~Aoshima.
\newblock Path integral networks: End-to-end differentiable optimal control.
\newblock {\em CoRR}, abs/1706.09597, 2017.

\bibitem{m2ac}
F.~Pan, J.~He, D.~Tu, and Q.~He.
\newblock Trust the model when it is confident: Masked model-based
  actor-critic.
\newblock In {\em Advances in Neural Information Processing Systems 33
  (NeurIPS'20)}, Virtual Conference, 2020.

\bibitem{trpo}
J.~Schulman, S.~Levine, P.~Abbeel, M.~I. Jordan, and P.~Moritz.
\newblock Trust region policy optimization.
\newblock In {\em Proceedings of the 32nd International Conference on Machine
  Learning (ICML'15)}, Lille, France, 2015.

\bibitem{shang2021mlj}
W.~Shang, Q.~Li, Z.~Qin, Y.~Yu, Y.~Meng, and J.~Ye.
\newblock Partially observable environment estimation with uplift inference for
  reinforcement learning based recommendation.
\newblock {\em Machine Learning}, 110(9):2603--2640, 2021.

\bibitem{mail}
J.~Shi, Y.~Yu, Q.~Da, S.~Chen, and A.~Zeng.
\newblock Virtual-taobao: Virtualizing real-world online retail environment for
  reinforcement learning.
\newblock In {\em Proceedings of the 33rd {AAAI} Conference on Artificial
  Intelligence (AAAI'19)}, Honolulu, Hawaii, 2019.

\bibitem{upn}
A.~Srinivas, A.~Jabri, P.~Abbeel, S.~Levine, and C.~Finn.
\newblock Universal planning networks: Learning generalizable representations
  for visuomotor control.
\newblock In {\em Proceedings of the 35th International Conference on Machine
  Learning (ICML'18)}, Stockholm, Sweden, 2018.

\bibitem{dyna-q}
R.~S. Sutton.
\newblock Integrated architectures for learning, planning, and reacting based
  on approximating dynamic programming.
\newblock In {\em Proceedings of the 7th International Conference on Machine
  Learning (ICML'90)}, Austin, Texas, 1990.

\bibitem{mujoco}
E.~Todorov, T.~Erez, and Y.~Tassa.
\newblock Mujoco: {A} physics engine for model-based control.
\newblock In {\em {IEEE/RSJ} International Conference on Intelligent Robots and
  Systems (IROS'20)}, Vilamoura, Portugal, 2012.

\bibitem{xu2021error}
T.~Xu, Z.~Li, and Y.~Yu.
\newblock Error bounds of imitating policies and environments for reinforcement
  learning.
\newblock {\em IEEE Transactions on Pattern Analysis and Machine Intelligence},
  2021.

\bibitem{yu2018sample}
Y.~Yu.
\newblock Towards sample efficient reinforcement learning.
\newblock In {\em Proceedings of the 27th International Joint Conference on
  Artificial Intelligence (IJCAI'18)}, Stockholm, Sweden, 2018.

\bibitem{gpm}
H.~Zhang, W.~Xu, and H.~Yu.
\newblock Generative planning for temporally coordinated exploration in
  reinforcement learning.
\newblock In {\em 10th International Conference on Learning Representations
  (ICLR'22)}, Virtual Conference, 2022.

\bibitem{zhu2022arxiv}
Z.-M. Zhu, X.-H. Chen, H.-L. Tian, K.~Zhang, and Y.~Yu.
\newblock Offline reinforcement learning with causal structured world models.
\newblock {\em arXiv preprint arXiv:2206.01474}, 2022.

\end{thebibliography}
\bibliographystyle{abbrv}

\clearpage
\appendix
\section{Planning Policy Iteration}
We first define the Bellman operator:
\begin{equation}
\mathcal{T}^\pi Q(s_t,\boldsymbol{\tau}_t^k)=\mathbb{E}_{p,r}\left[\sum_{m=0}^{k-1}\gamma^m r_{t+m}\right]+\gamma^k\mathbb{E}_{p}\left[ V(s_{t+k})\right],
\end{equation}
where
\begin{equation}
V(s_t) = \mathbb{E}_{\boldsymbol{\tau}_t^k\sim \boldsymbol{\pi}^k}\left[Q(s_{t},\boldsymbol{\tau}_t^k)\right],
\end{equation}
is the state value function. Then the proof of planning policy iteration will be given as follows.
\subsection{Lemma \ref{lemma1}}\label{apdx_lemma1}

\begin{apdx_lemma}[Planning Policy Evaluation]
Given any initial mapping $Q_0:\mathcal{S}\times\mathcal{A}^k\to \mathbb{R}$ with $|\mathcal{A}|<\infty$, update $Q_i$ to $Q_{i+1}$ with $Q_{i+1}=\mathcal{T}^\pi Q_i$ for all $i\in N$, $\{Q_i\}$ will converge to plan value of policy $\pi$ as $i\to\infty$.
\end{apdx_lemma}
\begin{proof}
For any $i\in N$, after updating $Q_i$ to $Q_{i+1}$ with $\mathcal{T}^\pi$, we have
\begin{align}
&Q_{i+1}(s_t,\boldsymbol{\tau}_t^k)=\mathcal{T}^\pi Q_k(s_t,\boldsymbol{\tau}_t^k)\\
=&\label{apdx_dp}\mathbb{E}_{p,r}\left[\sum_{m=0}^{k-1}\gamma^m r_{t+m}+\gamma^k \mathbb{E}_{\hat{\boldsymbol{\tau}}^k\sim \boldsymbol{\pi}^k}\left[Q_i(s_{t+k},\hat{\boldsymbol{\tau}}^k)\right]\right],
\end{align}
for any $s_t\in\mathcal{S}$ and $\boldsymbol{\tau}_t^k\in\mathcal{A}^k$. Then
\begin{equation}
\begin{aligned}
&\|Q_{i+1}-Q^\pi\|_\infty\\
=&\max_{s_t\in\mathcal{S},\boldsymbol{\tau}_t^k\in\mathcal{A}^k} 
 \left|Q_{i+1}(s_t,\boldsymbol{\tau}_t^k)-Q^\pi(s_t,\boldsymbol{\tau}_t^k)\right|\\
=&\max_{s_t\in\mathcal{S},\boldsymbol{\tau}_t^k\in\mathcal{A}^k} 
 \gamma^k \left|\mathbb{E}_{p,\boldsymbol{\pi}^k}\left[Q_i(s_{t+k},\hat{\boldsymbol{\tau}}^k)-Q^\pi(s_{t+k},\hat{\boldsymbol{\tau}}^k)\right]\right| \\
\leq &\max_{s_t\in\mathcal{S},\boldsymbol{\tau}_t^k\in\mathcal{A}^k} 
 \gamma^k \|Q_{i}-Q^\pi\|_\infty\\
=&\gamma^k \|Q_{i}-Q^\pi\|_\infty.
\end{aligned}
\end{equation}
So $Q_i=Q_\pi$ is a fixed point of this update rule, and the sequence $\{Q_i\}$ will converge to the plan value of policy $\pi$ as $i\to\infty$.
\end{proof}

\subsection{Lemma \ref{lemma2}}\label{apdx_lemma2}
\begin{apdx_lemma}[Planning Policy Improvement]
Given any mapping $\pi_{\mathrm{old}}\in\Pi:\mathcal{S}\to\Delta(\mathcal{A})$ with $|\mathcal{A}|<\infty$, update $\pi_\mathrm{old}$ to $\pi_{\mathrm{new}}$ with equation
\begin{equation}
    \pi_{\mathrm{new}} = \arg\max_{\pi\in\Pi} \sum_{\boldsymbol{\tau}_t^k\in\mathcal{A}^k}\boldsymbol{\pi}^k(\boldsymbol{\tau}_t^k|s_t)Q^{\pi_{\mathrm{old}}}(s_t,\boldsymbol{\tau}_t^k),
\end{equation} 
for each state $s_t$, then $Q^{\pi_{\mathrm{new}}}(s_t,\boldsymbol{\tau}_t^k)\geq Q^{\pi_{\mathrm{old}}}(s_t,\boldsymbol{\tau}_t^k)$, $\forall s_t\in\mathcal{S}$, $\boldsymbol{\tau}_t^k\in\mathcal{A}^k$.
\end{apdx_lemma}
\begin{proof}
By definition of the state value function, we have
\begin{equation}
\begin{aligned}
\label{a21}
V^{\pi_{\mathrm{old}}}(s_t) &= 
\sum_{\boldsymbol{\tau}_t^k\in\mathcal{A}^k}\boldsymbol{\pi}^k_{\mathrm{old}}(\boldsymbol{\tau}_t^k|s_t)Q^{\pi_{\mathrm{old}}}(s_t,\boldsymbol{\tau}_t^k) \\
&\leq \sum_{\boldsymbol{\tau}_t^k\in\mathcal{A}^k}\boldsymbol{\pi}^k_{\mathrm{new}}(\boldsymbol{\tau}_t^k|s_t)Q^{\pi_{\mathrm{old}}}(s_t,\boldsymbol{\tau}_t^k),
\end{aligned}
\end{equation}
for any $s_t\in\mathcal{S}$. Reapplying the inequality \eqref{a21}, the result is given by
\begin{equation}
\begin{aligned}
&Q^{\pi_{\mathrm{old}}}(s_t,\boldsymbol{\tau}_t^k)
=\mathbb{E}_{p,r}\left[\sum_{m=0}^{k-1}\gamma^m r_{t+m}+\gamma^k V^{\pi_{\mathrm{old}}}(s_{t+k})\right]\\
\leq& \mathbb{E}_{p,r}\left[\sum_{m=0}^{k-1}\gamma^m r_{t+m}+\gamma^k \sum_{\boldsymbol{\tau}^k\in\mathcal{A}^k}\boldsymbol{\pi}^k_{\mathrm{new}}(\boldsymbol{\tau}^k|s_{t+k})Q^{\pi_{\mathrm{old}}}(s_{t+k},\boldsymbol{\tau}^k)\right]\\
=& \mathbb{E}_{p,r,\pi_{\mathrm{new}}}\left[\sum_{m=0}^{2k-1}\gamma^m r_{t+m}+\gamma^{2k} V^{\pi_{\mathrm{old}}}(s_{t+2k})\right]\\
\vdots&\\
=& \mathbb{E}_{p,r,\pi_{\mathrm{new}}}\left[\sum_{m=0}^{\infty}\gamma^m r_{t+m}\right]=Q^{\pi_{\mathrm{new}}}(s_t,\boldsymbol{\tau}_t^k),
\end{aligned}
\end{equation}
$\forall s_t\in\mathcal{S}$, $\boldsymbol{\tau}_t^k\in\mathcal{A}^k$.
\end{proof}

\subsection{Theorem \ref{thm1}}\label{apdx_thm1}
\begin{apdx_thm}[Planning Policy Iteration]
Given any initial mapping $\pi_0\in\Pi:\mathcal{S}\to\Delta(\mathcal{A})$ with $|\mathcal{A}|<\infty$, compute corresponding $Q^{\pi_i}$ in planning policy evaluation step and update $\pi_i$ to $\pi_{i+1}$ in planning policy improvement step for all $i\in N$, $\{\pi_i\}$ will converge to the optimal policy $\pi_*$ that $Q^{\pi_*}(s_t,\boldsymbol{\tau}_t^k)\geq Q^{\pi}(s_t,\boldsymbol{\tau}_t^k)$, $\forall s_t\in\mathcal{S}$, $\boldsymbol{\tau}_t^k\in\mathcal{A}^k$, and $\pi\in\Pi$.
\end{apdx_thm}
\begin{proof}
The sequence $\{\pi_i\}$ will converge to some $\pi_*$ since $\{Q^{\pi_i}\}$ is monotonically increasing with $i$ and bounded above $\pi\in\Pi$. We next show that $\pi_*$ is the optimal policy indeed. By Lemma 2\ref{apdx_lemma2}, at convergence, we can't find any $\pi\in\Pi$ satisfying
\begin{equation}
V^{\pi_*}(s_t)<\sum_{\boldsymbol{\tau}_t^k\in\mathcal{A}^k}\boldsymbol{\pi}^k(\boldsymbol{\tau}_t^k|s_t)Q^{\pi_*}(s_t,\boldsymbol{\tau}_t^k),
\end{equation}
for any $s_t\in\mathcal{S}$. In other words, we have
\begin{equation}
\label{a31}
V^{\pi_*}(s_t)\geq \sum_{\boldsymbol{\tau}_t^k\in\mathcal{A}^k}\boldsymbol{\pi}^k(\boldsymbol{\tau}_t^k|s_t)Q^{\pi_{*}}(s_t,\boldsymbol{\tau}_t^k),
\end{equation}
for any $s_t\in\mathcal{S}$ and $\pi\in\Pi$. Then we obtain
\begin{equation}
\begin{aligned}
&Q^{\pi_{*}}(s_t,\boldsymbol{\tau}_t^k)
=\mathbb{E}_{p,r}\left[\sum_{m=0}^{k-1}\gamma^m r_{t+m}+\gamma^k V^{\pi_{*}}(s_{t+k})\right]\\
\geq& \mathbb{E}_{p,r}\left[\sum_{m=0}^{k-1}\gamma^m r_{t+m}+\gamma^k \sum_{\boldsymbol{\tau}^k\in\mathcal{A}^k}\boldsymbol{\pi}^k(\boldsymbol{\tau}^k|s_{t+k})Q^{\pi_{*}}(s_{t+k},\boldsymbol{\tau}^k)\right]\\
=& \mathbb{E}_{p,r,\pi}\left[\sum_{m=0}^{2k-1}\gamma^m r_{t+m}+\gamma^{2k} V^{\pi_{*}}(s_{t+2k})\right]\\
\vdots&\\
=& \mathbb{E}_{p,r,\pi}\left[\sum_{m=0}^{\infty}\gamma^m r_{t+m}\right]=Q^{\pi}(s_t,\boldsymbol{\tau}_t^k),
\end{aligned}
\end{equation}
$\forall s_t\in\mathcal{S}$, $\boldsymbol{\tau}_t^k\in\mathcal{A}^k$, and $\pi\in\Pi$. Hence $\pi_*$ is the optimal policy in $\Pi$.
\end{proof}

\section{Soft Planning Policy Iteration}\label{softppi}
We define the soft plan value function with entropy as
\begin{equation}
\begin{aligned}
&Q^\pi(s_t,\boldsymbol{\tau}_t^k)=Q^\pi(s_t,a_t,a_{t+1},...,a_{t+k-1})\\
=&\mathbb{E}_{p,r,\pi}\left[\sum_{m=0}^{k-1}\left(\gamma^m r_{t+m}-\log\pi(a_{t+m}|s_{t+m})\right)\right.\\
&~~~~+\left.\gamma^k \sum_{m=0}^{\infty}\left(\gamma^m r_{t+k+m}-\log \pi(a_{t+k+m}|s_{t+k+m})\right)\right].
\end{aligned}
\end{equation}
Then extended soft Bellman operator can be written as
\begin{equation}
\mathcal{T}^\pi Q(s_t,\boldsymbol{\tau}_t^k)=\mathbb{E}_{p,r}\left[\sum_{m=0}^{k-1}\gamma^m r_{t+m}\right]+\gamma^k\mathbb{E}_{p}\left[ V(s_{t+k})\right],
\end{equation}
where
\begin{equation}
V(s_t) = \mathbb{E}_{\boldsymbol{\tau}_t^k\sim \boldsymbol{\pi}^k}\left[Q(s_{t},\boldsymbol{\tau}_t^k)
-\sum_{m=0}^{k-1}\log\pi(a_{t+m}|s_{t+m})\right],
\end{equation}
is the soft state value function. The proof of soft planning policy iteration will be given as follows.

\subsection{Lemma 3}\label{apdx_lemma3}
\begin{apdx_lemma}[Soft Planning Policy Evaluation]
Given any initial mapping $Q_0:\mathcal{S}\times\mathcal{A}^k\to \mathbb{R}$ with $|\mathcal{A}|<\infty$, update $Q_i$ to $Q_{i+1}$ with $Q_{i+1}=\mathcal{T}^\pi Q_i$ for all $i\in N$, $\{Q_i\}$ will converge to soft plan value of policy $\pi$ as $i\to\infty$.
\end{apdx_lemma}
\begin{proof}
Define $k$-step reward $r^\pi(s_t, \boldsymbol{\tau}_t^k)$ with entropy augmentation as
\begin{equation}
\mathbb{E}_{p,r,\pi}\left[\sum_{m=0}^{k-1}\left(\gamma^m r_{t+m}-\gamma^k\log\pi(a_{t+k+m}|s_{t+k+m})\right)\right],
\end{equation}
and rewrite the update rule as
\begin{equation}
Q_{i+1}(s_t,\boldsymbol{\tau}_t^k)\leftarrow r^\pi(s_t, \boldsymbol{\tau}_t^k)+\gamma^k\mathbb{E}_{p,\boldsymbol{\pi}^k}\left[ Q_i(s_{t+k},\boldsymbol{\tau}_{t+k}^k)\right].
\end{equation}
Then applying the standard convergence result for planning policy evaluation given by Lemma 1\ref{apdx_lemma1}, we can obtain that $\{Q_i\}$ will converge to the soft plan value of policy $\pi$ as $i\to\infty$.
\end{proof}

\subsection{Lemma 4}\label{apdx_lemma4}
\begin{apdx_lemma}[Soft Planning Policy Improvement]
Given any mapping $\pi_{\mathrm{old}}\in\Pi:\mathcal{S}\to\Delta(\mathcal{A})$ with $|\mathcal{A}|<\infty$, update $\pi_\mathrm{old}$ to $\pi_{\mathrm{new}}$ with equation
\begin{equation}
\pi_{\mathrm{new}}=\arg \min_{\pi\in\Pi}\mathrm{D}_{\mathrm{KL}}\left(\boldsymbol{\pi}^k(\cdot|s_t)\bigg\|\frac{\mathrm{exp}(Q^{\pi_{\mathrm{old}}}(s_t,\cdot))}{Z^{\pi_{\mathrm{old}}}(s_t)}\right),
\end{equation} 
for each state $s_t$, where $Z^{\pi_{\mathrm{old}}}(s_t)$ normalizes the distribution, then $Q^{\pi_{\mathrm{new}}}(s_t,\boldsymbol{\tau}_t^k)\geq Q^{\pi_{\mathrm{old}}}(s_t,\boldsymbol{\tau}_t^k)$, $\forall s_t\in\mathcal{S}$, $\boldsymbol{\tau}_t^k\in\mathcal{A}^k$.
\end{apdx_lemma}
\begin{proof}
Define
\begin{equation}
J_{\pi_{\mathrm{old}}}(\boldsymbol{\pi}^k(\cdot|s_t))
=\mathrm{D}_{\mathrm{KL}}\left(\boldsymbol{\pi}^k(\cdot|s_t)\bigg\|\frac{\mathrm{exp}(Q^{\pi_{\mathrm{old}}}(s_t,\cdot))}{Z^{\pi_{\mathrm{old}}}(s_t)}\right),
\end{equation}
then it must be the case that
\begin{equation}J_{\pi_{\mathrm{old}}}(\boldsymbol{\pi}^k_{\mathrm{new}}(\cdot|s_t))\leq J_{\pi_{\mathrm{old}}}(\boldsymbol{\pi}^k_{\mathrm{old}}(\cdot|s_t)),
\end{equation}
since we choose $\pi_{\mathrm{new}}=\arg \min_{\pi\in\Pi} J_{\pi_{\mathrm{old}}}(\boldsymbol{\pi}^k(\cdot|s_t))$.
The inequality reduces to
\begin{equation}
\begin{aligned}
\label{a41}
&\mathbb{E}_{\boldsymbol{\tau}_t^k\sim\boldsymbol{\pi}^k_{\mathrm{new}}}
\left[Q^{\pi_{\mathrm{old}}}(s_t,\boldsymbol{\tau}_t^k)
-\log\boldsymbol{\pi}^k_{\mathrm{new}}(\boldsymbol{\tau}_t^k|s_t)\right]\\
\geq & \mathbb{E}_{\boldsymbol{\tau}_t^k\sim\boldsymbol{\pi}^k_{\mathrm{old}}}
\left[Q^{\pi_{\mathrm{old}}}(s_t,\boldsymbol{\tau}_t^k)
-\log \boldsymbol{\pi}^k_{\mathrm{old}}(\boldsymbol{\tau}_t^k|s_t)\right]=V^{\pi_{\mathrm{old}}}(s_t),
\end{aligned}
\end{equation}
since $Z^{\pi_{\mathrm{old}}}(s_t)$ only depends on the state. For convenience, we further define 
\begin{equation}
G_{\pi_{\mathrm{old}}}(\boldsymbol{\pi}^k(\cdot|s_t))=\mathbb{E}_{\boldsymbol{\tau}_t^k\sim\boldsymbol{\pi}^k}
\left[Q^{\pi_{\mathrm{old}}}(s_t,\boldsymbol{\tau}_t^k)-\log \boldsymbol{\pi}^k(\boldsymbol{\tau}_t^k|s_t)\right],
\end{equation}
and rewrite the inequality \eqref{a41} as
\begin{equation}
\label{a42}
G_{\pi_{\mathrm{old}}}(\boldsymbol{\pi}_{\mathrm{new}}^k(\cdot|s_t))\geq V^{\pi_{\mathrm{old}}}(s_t).
\end{equation}
Reapplying the inequality \eqref{a42}, the result is given by
\begin{equation}
\begin{aligned}
&Q^{\pi_{\mathrm{old}}}(s_t,\boldsymbol{\tau}_t^k)
=\mathbb{E}_{p,r}\left[\sum_{m=0}^{k-1}\gamma^m r_{t+m}+\gamma^k V^{\pi_{\mathrm{old}}}(s_{t+k})\right]\\
\leq& \mathbb{E}_{p,r}\left[\sum_{m=0}^{k-1}\gamma^m r_{t+m}+\gamma^k G_{\pi_{\mathrm{old}}}(\boldsymbol{\pi}_{\mathrm{new}}^k(\cdot|s_{t+k}))\right]\\
=& \mathbb{E}_{p,r,\pi_{\mathrm{new}}}\left[r^{\pi_{\mathrm{new}}}(s_t, \boldsymbol{\tau}_t^k)+\sum_{m=k}^{2k-1}\gamma^m r_{t+m}+\gamma^{2k} V^{\pi_{\mathrm{old}}}(s_{t+2k})\right]\\
\vdots&\\
=& \mathbb{E}_{p,r,\pi_{\mathrm{new}}}\left[\sum_{m=0}^{\infty}\left(\gamma^m r_{t+m}-\log\pi(a_{t+m}|s_{t+m})\right)\right]\\
=&Q^{\pi_{\mathrm{new}}}(s_t,\boldsymbol{\tau}_t^k),
\end{aligned}
\end{equation}
$\forall s_t\in\mathcal{S}$, $\boldsymbol{\tau}_t^k\in\mathcal{A}^k$.
\end{proof}

\begin{table*}[t!]
\small
\centering
\begin{tabular}{|c|c|c|c|c|c|c|}
\hline
\multirow{2}{*}{environment} & Inverted & \multirow{2}{*}{Hopper} & \multirow{2}{*}{Swimmer} & \multirow{2}{*}{Half Cheetah} & \multirow{2}{*}{Walker2d} & \multirow{2}{*}{Ant}\\
\multirow{2}{*}{} & Pendulum & \multirow{2}{*}{} & \multirow{2}{*}{} & \multirow{2}{*}{} & \multirow{2}{*}{} & \multirow{2}{*}{} \\
\hline
\multirow{2}{*}{steps} & \multirow{2}{*}{10k} & \multirow{2}{*}{100k} & \multicolumn{3}{c|}{\multirow{2}{*}{200k}} & \multirow{2}{*}{300k} \\
\multirow{2}{*}{} & \multirow{2}{*}{} & \multirow{2}{*}{} & \multicolumn{3}{c|}{\multirow{2}{*}{}} & \multirow{2}{*}{} \\

\hline
\multirow{2}{*}{Update-To-Date ratio} & \multicolumn{6}{c|}{\multirow{2}{*}{1 for actor, 20 for critic}} \\
\multirow{2}{*}{} & \multicolumn{6}{c|}{\multirow{2}{*}{}}\\

\hline
\multirow{2}{*}{plan length $k$} & \multicolumn{3}{c|}{\multirow{2}{*}{3}} & \multicolumn{3}{c|}{\multirow{2}{*}{2}} \\
\multirow{2}{*}{}& \multicolumn{3}{c|}{\multirow{2}{*}{}} & \multicolumn{3}{c|}{\multirow{2}{*}{}}\\

\hline
\multirow{2}{*}{model rollout schedule} & 1$\rightarrow$5 over & 1$\rightarrow$4 over & \multirow{2}{*}{1} & 1$\rightarrow$4 over & \multirow{2}{*}{1} & 1$\rightarrow$20 over \\
\multirow{2}{*}{} & 0$\rightarrow$1k & 20k$\rightarrow$50k & \multirow{2}{*}{} & 20k$\rightarrow$80k &\multirow{2}{*}{} & 20k$\rightarrow$150k \\

\hline
\multirow{2}{*}{target entropy} & \multirow{2}{*}{-0.05} & \multicolumn{2}{c|}{\multirow{2}{*}{-1}} & \multicolumn{2}{c|}{\multirow{2}{*}{-3}} & \multirow{2}{*}{-4} \\
\multirow{2}{*}{} & \multirow{2}{*}{} & \multicolumn{2}{c|}{\multirow{2}{*}{}} & \multicolumn{2}{c|}{\multirow{2}{*}{}} & \multirow{2}{*}{} \\
\hline

\end{tabular}
\caption{Hyper-parameter settings for MPPVE results presented in Figure \ref{fig3}. $x\rightarrow y$ over $a\rightarrow b$ denotes a thresholded linear increasing schedule, \textit{i.e.} the length of model rollouts at step $t$ is calculated by $f(t)=\min \left(\max \left(x+\frac{t-a}{b-a} \cdot(y-x), x\right), y\right)$.}
\label{hyperparam}
\end{table*}

\subsection{Theorem 2}\label{apdx_thm2}
\begin{apdx_thm}[Soft Planning Policy Iteration]
Given any initial mapping $\pi_0\in\Pi:\mathcal{S}\to\Delta(\mathcal{A})$ with $|\mathcal{A}|<\infty$, compute corresponding $Q^{\pi_i}$ in soft planning policy evaluation step and update $\pi_i$ to $\pi_{i+1}$ in soft planning policy improvement step for all $i\in N$, $\{\pi_i\}$ will converge to the optimal policy $\pi_*$ that $Q^{\pi_*}(s_t,\boldsymbol{\tau}_t^k)\geq Q^{\pi}(s_t,\boldsymbol{\tau}_t^k)$, $\forall s_t\in\mathcal{S}$, $\boldsymbol{\tau}_t^k\in\mathcal{A}^k$, and $\pi\in\Pi$.
\end{apdx_thm}
\begin{proof}
The sequence $\{\pi_i\}$ will converge to some $\pi_*$ since $\{Q^{\pi_i}\}$ is monotonically increasing with $i$ and bounded above $\pi\in\Pi$. We next show that $\pi_*$ is the optimal policy indeed. By Lemma 4\ref{apdx_lemma4}, at convergence, we can't find any $\pi\in\Pi$ satisfying
\begin{equation}
V^{\pi_*}(s_t)<\mathbb{E}_{\boldsymbol{\tau}_t^k\sim\boldsymbol{\pi}^k}
\left[Q^{\pi_{*}}(s_t,\boldsymbol{\tau}_t^k)
-\log\boldsymbol{\pi}^k(\boldsymbol{\tau}_t^k|s_t)\right],
\end{equation}
for any $s_t\in\mathcal{S}$. In other words, we have
\begin{equation}
\label{a61}
V^{\pi_*}(s_t)\geq \mathbb{E}_{\boldsymbol{\tau}_t^k\sim\boldsymbol{\pi}^k}
\left[Q^{\pi_{*}}(s_t,\boldsymbol{\tau}_t^k)
-\log\boldsymbol{\pi}^k(\boldsymbol{\tau}_t^k|s_t)\right],
\end{equation}
which can be rewritten as
\begin{equation}
V^{\pi_*}(s_t)\geq G_{\pi_{*}}(\boldsymbol{\pi}^k(\cdot|s_t)),
\end{equation}
for any $s_t\in\mathcal{S}$ and $\pi\in\Pi$. Then we obtain
\begin{equation}
\begin{aligned}
&Q^{\pi_{*}}(s_t,\boldsymbol{\tau}_t^k)
=\mathbb{E}_{p,r}\left[\sum_{m=0}^{k-1}\gamma^m r_{t+m}+\gamma^k V^{\pi_{*}}(s_{t+k})\right]\\
\geq& \mathbb{E}_{p,r}\left[\sum_{m=0}^{k-1}\gamma^m r_{t+m}+\gamma^k G_{\pi_{*}}(\boldsymbol{\pi}^k(\cdot|s_t))\right]\\
=& \mathbb{E}_{p,r,\pi}\left[r^\pi(s_t, \boldsymbol{\tau}_t^k)+\sum_{m=k}^{2k-1}\gamma^m r_{t+m}+\gamma^{2k} V^{\pi_{*}}(s_{t+2k})\right]\\
\vdots&\\
=& \mathbb{E}_{p,r,\pi}\left[\sum_{m=0}^{\infty}\left(\gamma^m r_{t+m}-\log\pi(a_{t+m}|s_{t+m})\right)\right]\\
=&Q^{\pi}(s_t,\boldsymbol{\tau}_t^k),
\end{aligned}
\end{equation}
$\forall s_t\in\mathcal{S}$, $\boldsymbol{\tau}_t^k\in\mathcal{A}^k$, and $\pi\in\Pi$. Hence $\pi_*$ is the optimal policy in $\Pi$.
\end{proof}

\section{Soft MPPVE}
\label{softMPPVE}
\subsection{Model Learning}
The same as the vanilla MPPVE.

\subsection{Soft Multi-step Plan Value Estimation}
Soft plan value is estimated via minimizing the expected multi-step temporal-difference error:
\begin{equation}
\label{softQobject}
J_Q(\psi) = \mathbb{E}_{(s_t,\boldsymbol{\tau}^k_t,\boldsymbol{r}^k_t,s_{t+k})\sim \mathcal{D}_{\mathrm{env}}\cup \mathcal{D}_{\mathrm{model}}}\left[l_{\mathrm{TD}}(s_t,\boldsymbol{\tau}^k_t,\boldsymbol{r}^k_t,s_{t+k})\right],
\end{equation}
with
\begin{equation}
\boldsymbol{r}^k_t = (r_t,r_{t+1},...,r_{t+k-1}),
\end{equation}
\begin{equation}
l_{\mathrm{TD}}(s_t,\boldsymbol{\tau}^k_t,\boldsymbol{r}^k_t,s_{t+k})=\frac{1}{2}\left(Q_\psi(s_t,\boldsymbol{\tau}^k_t)
-y_{\psi^-}(\boldsymbol{r}^k_t,s_{t+k})\right)^2,
\end{equation}
\begin{equation}
\begin{aligned}
y_{\psi^-}&(\boldsymbol{r}^k_t,s_{t+k})=\sum_{m=0}^{k-1}\gamma^m r_{t+m}\\
&+\gamma^k \mathbb{E}_{\hat{\boldsymbol{\tau}}^k}\left[
Q_{\psi^-}(s_{t+k},\hat{\boldsymbol{\tau}}^k)
-\alpha \log \boldsymbol{\pi}^k_{\phi,\theta}(\hat{\boldsymbol{\tau}}^k|s_{t+k})\right],
\end{aligned}
\end{equation}
where $\psi^-$ is the parameters of the target network, and $\alpha$ is the weight of entropy. The gradient of Eq.\eqref{softQobject} can be estimated without bias by
\begin{equation}
\hat{\nabla}_\psi J_Q(\psi)=
\left(Q_\psi(s_t,\boldsymbol{\tau}^k_t)-\hat{y}_{\psi^-}(\boldsymbol{r}^k_t,s_{t+k})\right) \nabla_\psi Q_\psi(s_t,\boldsymbol{\tau}^k_t),
\end{equation}
with
\begin{equation}
\begin{aligned}
\hat{y}_{\psi^-}&(\boldsymbol{r}^k_t,s_{t+k})=\sum_{m=0}^{k-1}\gamma^m r_{t+m}\\
&+\gamma^k \left(Q_{\psi^-}(s_{t+k},\boldsymbol{\tau}_{t+k}^k)-\alpha \log \boldsymbol{\pi}^k_{\phi,\theta}(\boldsymbol{\tau}_{t+k}^k|s_{t+k})\right),
\end{aligned}
\end{equation}
where $(s_t,\boldsymbol{\tau}^k_t,\boldsymbol{r}^k_t,s_{t+k})$ is sampled from the replay buffer and $\boldsymbol{\tau}_{t+k}^k$ is sampled according to current soft planning policy.

\subsection{Model-based Soft Planning Policy Improvement}
Model-based soft planning policy with maximum entropy can be trained by minimizing
\begin{equation}
\label{softpitarget}
    J_{\boldsymbol{\pi}^k}(\phi)=\mathbb{E}_{s_t\sim\mathcal{D}_{\mathrm{env}}}\left[\mathcal{J}(s_t)\right],
\end{equation}
with 
\begin{equation}
\mathcal{J}(s_t)=\mathbb{E}_{\boldsymbol{\tau}^k_t\sim\boldsymbol{\pi}^k_{\phi,\theta}(\cdot|s_t)}
    \left[-Q_\psi(s_t,\boldsymbol{\tau}^k_t)+\alpha \log \boldsymbol{\pi}^k_{\phi,\theta}(\boldsymbol{\tau}^k|s_{t})\right].
\end{equation}
The gradient of \eqref{softpitarget} can be approximated by
\begin{equation}
\nabla_{\boldsymbol{\tau}^k_t}\left(\log \boldsymbol{\pi}^k_{\phi,\theta}(\boldsymbol{\tau}^k_t|s_{t})-Q_\psi(s_t,\boldsymbol{\tau}^k_t)\right) \nabla_\phi \boldsymbol{f}_{\phi,\theta}(s_t,\boldsymbol{\eta^k_t}),
\end{equation}
where $\boldsymbol{f}_{\phi,\theta}(s_t,\boldsymbol{\eta^k_t})=\boldsymbol{\tau}^k_t$ is the reparameterized neural network transformation of planning policy with $\boldsymbol{\eta}_t^k=(\eta_1,\eta_2,...,\eta_k)$, which is a $k$-step noise vector sampled from Gaussian distribution.

\section{Hyper-parameter Settings}
\label{hyper}
See Table \ref{hyperparam}.

\section{Ablation Study}
We carry out an ablation study with three algorithms to verify that MPPVE’s gain is due to the model-based planning policy improvement based on plan value estimation. 

\begin{figure*}[t!]
    \centering
    \subfigure[Performance]{
        \centering
        \includegraphics[width=0.24\linewidth]{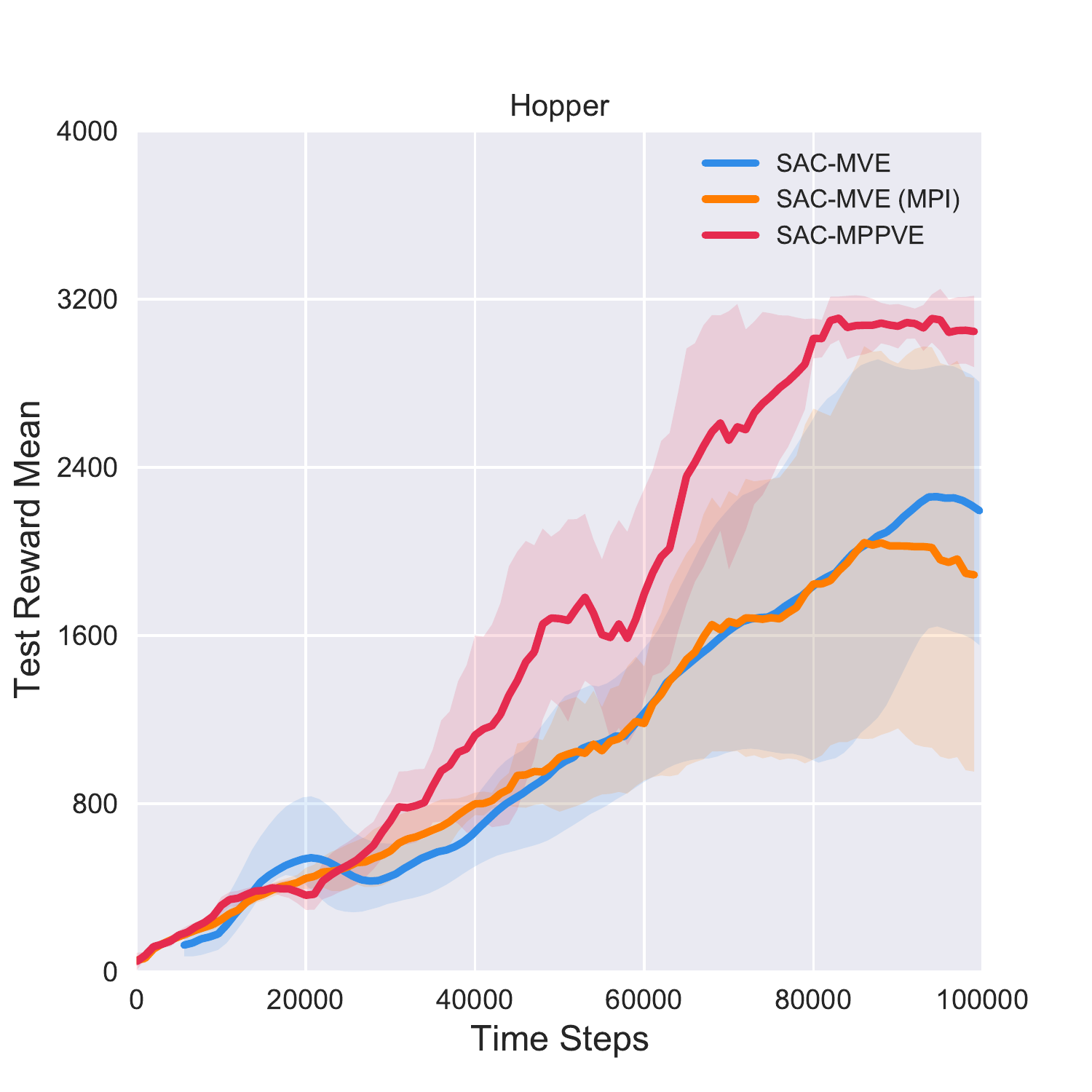}
        \label{fig6(a)}
    }%
    \subfigure[Policy evaluation]{
        \centering
        \includegraphics[width=0.24\linewidth]{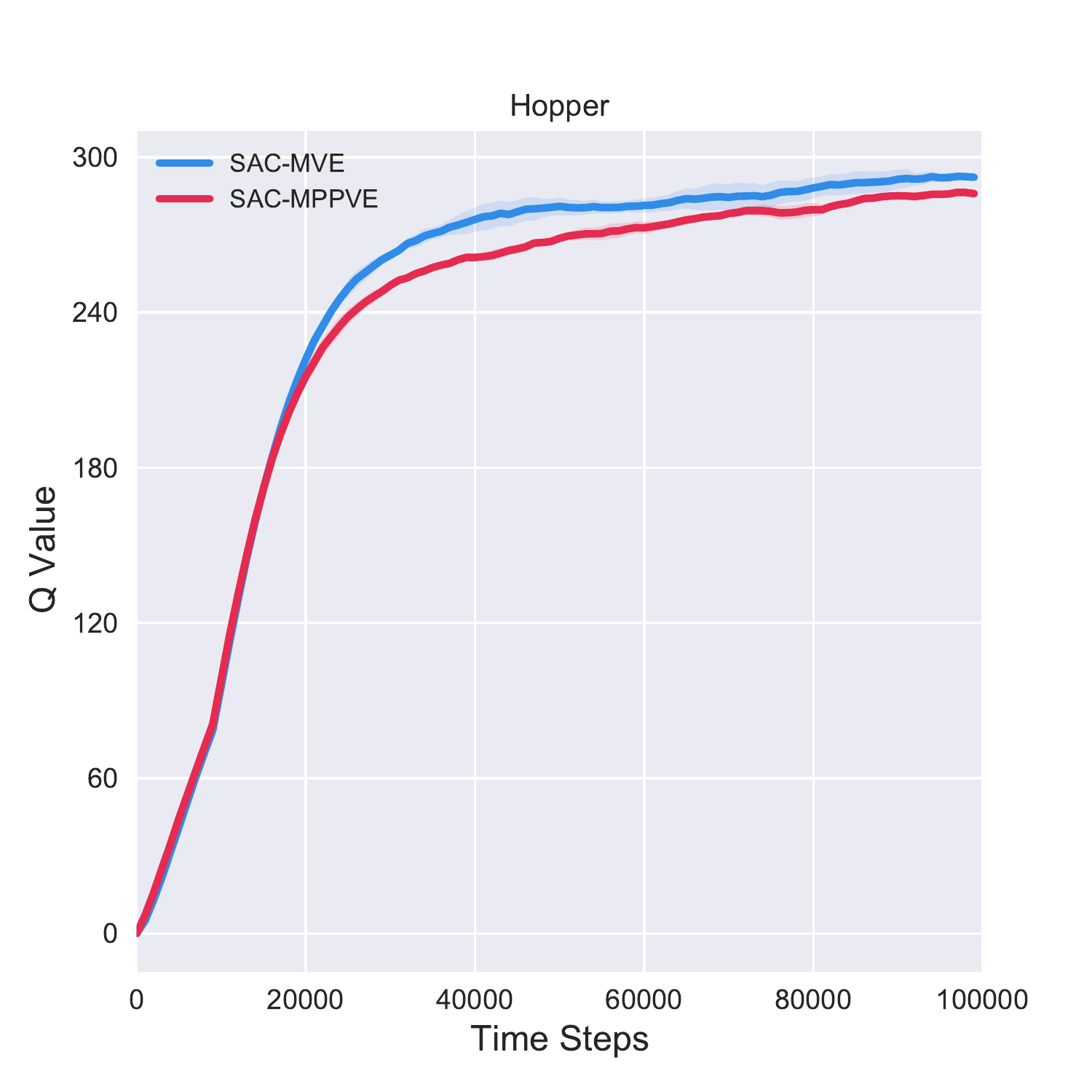}
        \label{fig6(b)}
    }%
    \subfigure[Average of value bias]{
        \centering
        \includegraphics[width=0.24\linewidth]{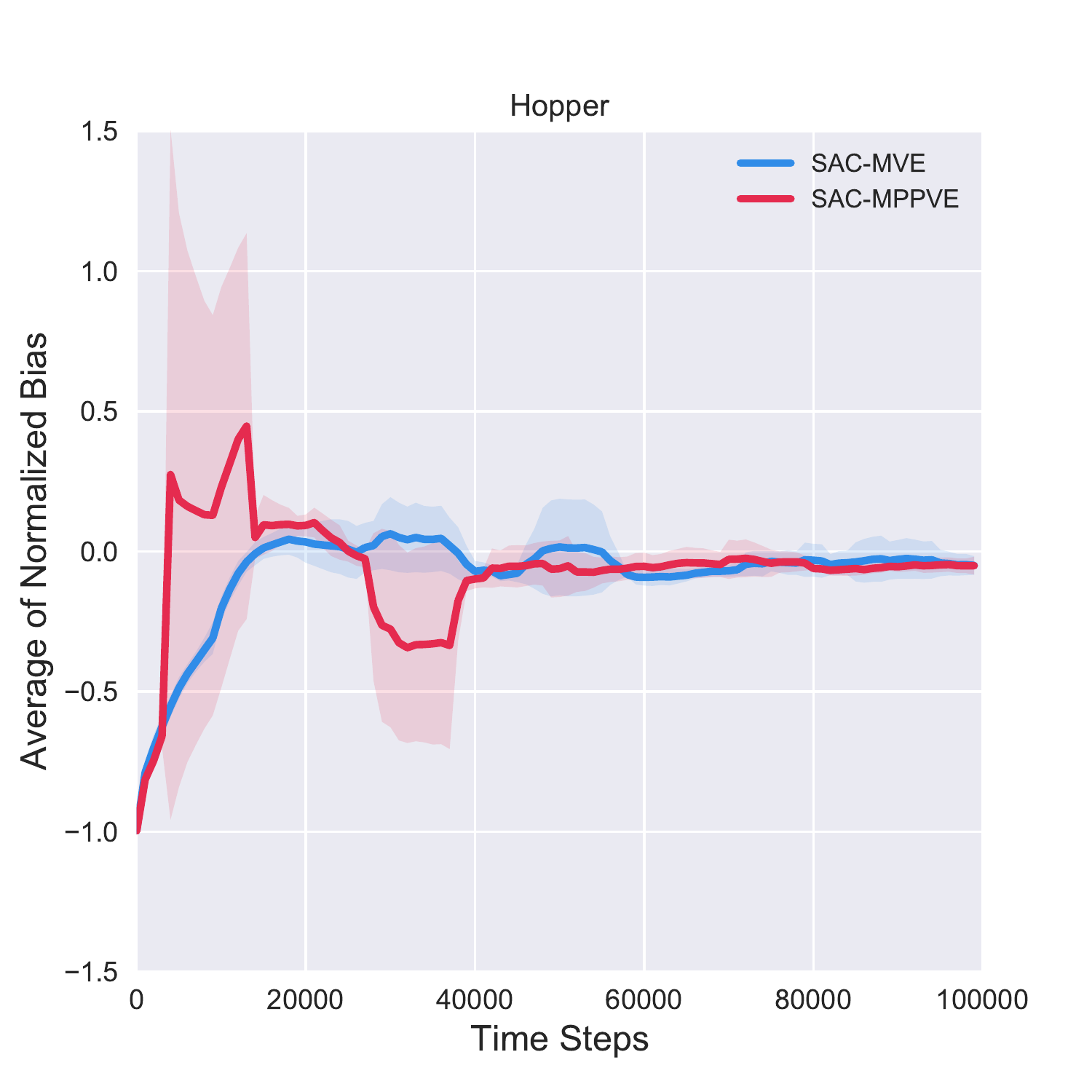}
        \label{fig6(c)}
    }%
    \subfigure[Std of value bias]{
        \centering
        \includegraphics[width=0.24\linewidth]{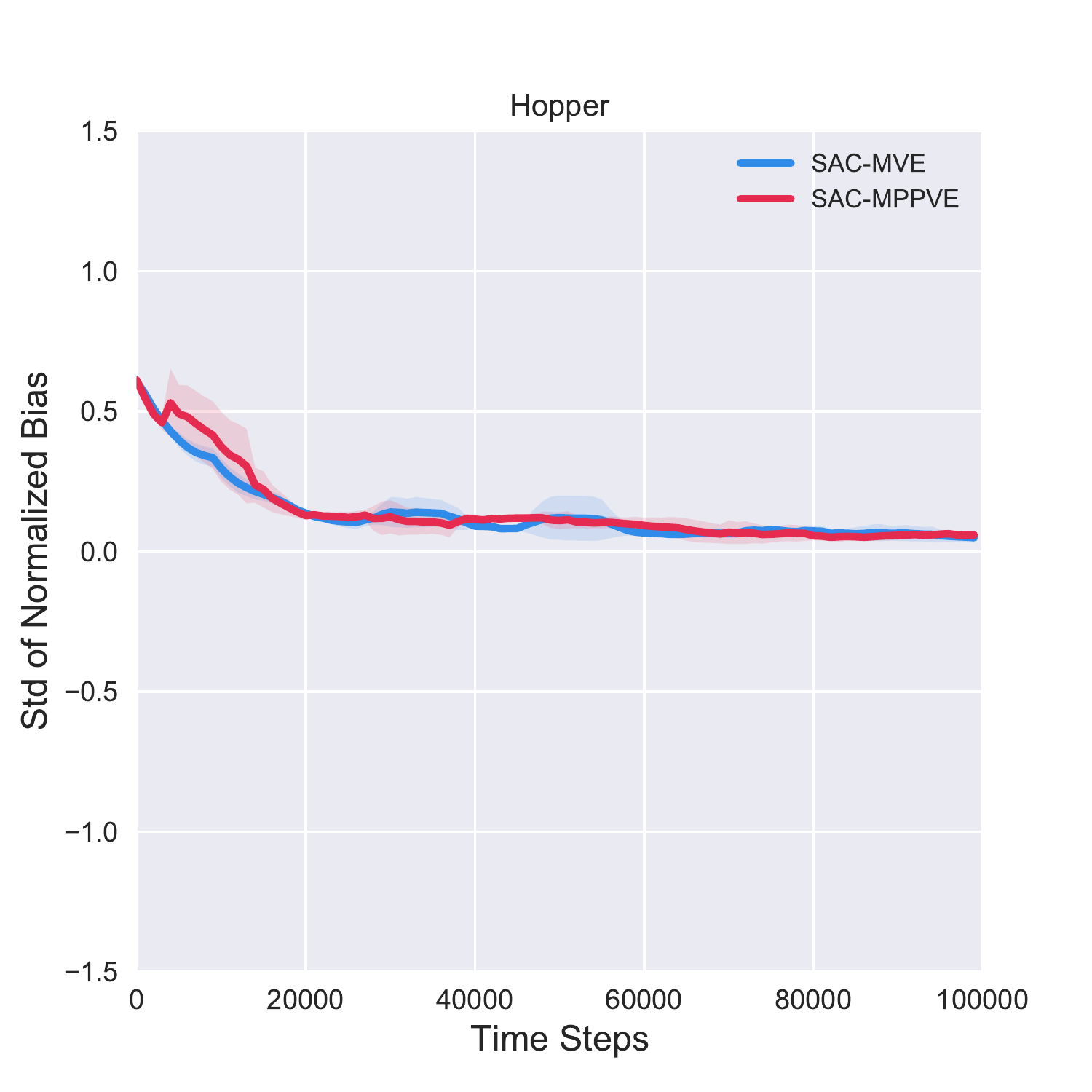}
        \label{fig6(d)}
    }%
    \centering
    \caption{Ablation study for MPPVE on Hopper-v3 task, through comparison of SAC-MPPVE (degraded version of MPPVE), SAC-MVE, and SAC-MVE (MPI). (a) Evaluation of episodic reward during the learning process. (b) Learning value estimation for the same fixed policy, where SAC-MPPVE estimates the plan value while SAC-MVE estimates the action value. (c) Average of the bias between neural value estimation and true value evaluation over state-action(s) space. (d) Standard error of the bias between neural value estimation and true value evaluation over state-action(s) space.}
    \label{fig6}
\end{figure*}

\subsubsection{SAC-MPPVE}
SAC-MPPVE builds on SAC but replaces its original value estimation and policy improvement with our proposed plan value estimation and model-based planning policy improvement, respectively. See Algorithm \ref{sac-MPPVE}.

\begin{algorithm}[h!]
\caption{SAC-MPPVE}
\label{sac-MPPVE}
\textbf{Input}: Initial dynamics model $p_\theta(s^\prime,r|s,a)$, policy $\pi_\phi(a|s)$, plan value function $Q_\psi(s,\boldsymbol{\tau}^k)$ and its target network $Q_{\psi^-}(s,\boldsymbol{\tau}^k)$, plan length $k$, environment buffer $\mathcal{D}_{\mathrm{env}}$, start size $U$, batch size $B$, and learning rate $\lambda_Q$, $\lambda_\pi$.

\begin{algorithmic}[1] 
\STATE Explore in the environment for $U$ steps and add data to $\mathcal{D}_{\mathrm{env}}$
\FOR{$N$ epochs}
\STATE Train model $p_\theta$ on $\mathcal{D}_{\mathrm{env}}$ by maximizing Eq.\eqref{pobject}
\FOR{$E$ steps}
\STATE Sample action to perform in the environment according to $\pi_\phi$; add the environmental transition to $\mathcal{D}_{\mathrm{env}}$
\STATE Sample $B$ $k$-step trajectories from $\mathcal{D}_{\mathrm{env}}$  \STATE Update $Q_\psi$ via $\psi\leftarrow\psi-\lambda_Q\hat{\nabla}_\psi J_{Q}(\psi)$ by Eq.\eqref{Qgrad}
\STATE Update target critic via $\psi^- \leftarrow \tau\psi+(1-\tau)\psi^-$
\STATE Update $\pi_\phi$ via $\phi \leftarrow \phi- \lambda_\pi \hat{\nabla}_\phi J_{\boldsymbol{\pi}^k}(\phi)$ by Eq.\eqref{mspg}
\ENDFOR
\ENDFOR
\end{algorithmic}
\end{algorithm}

\subsubsection{SAC-MVE}
SAC-MVE applies the MVE (Model-based Value Expansion) technique to SAC, which forms $H$-step TD targets by unrolling the model dynamics for $H$ steps. We refer the reader to \cite{mve} for further discussion of this technique. See Algorithm \ref{sac-mve}.

\begin{algorithm}[h!]
\caption{SAC-MVE}
\label{sac-mve}
\textbf{Input}: Initial dynamics model $p_\theta(s^\prime,r|s,a)$, policy $\pi_\phi(a|s)$, action-value $Q_\psi(s,a)$ and its target network $Q_{\psi^-}(s,a)$, rollout length $k$, environment buffer $\mathcal{D}_{\mathrm{env}}$, start size $U$, batch size $B$.

\begin{algorithmic}[1] 
\STATE Explore in the environment for $U$ steps and add data to $\mathcal{D}_{\mathrm{env}}$
\FOR{$N$ epochs}
\STATE Train model $p_\theta$ on $\mathcal{D}_{\mathrm{env}}$ by maximizing Eq.\eqref{pobject}
\FOR{$E$ steps}
\STATE Sample action to perform in the environment according to $\pi_\phi$; add the environmental transition to $\mathcal{D}_{\mathrm{env}}$
\STATE Sample transitions $\tau_{real}$ from $\mathcal{D}_{\mathrm{env}}$ and make $k$-step rollouts to get $\tau_{fake}$
\STATE Update critic $Q_\psi$ on \{$\tau_{real} \cup \tau_{fake}$\} by minimizing MVE \cite{mve} error
\STATE Update target critic via $\psi^- \leftarrow \tau\psi+(1-\tau)\psi^-$
\STATE Update policy $\pi_\phi$ on $\tau_{real}$ via policy gradients given by critic $Q_\psi$
\ENDFOR
\ENDFOR
\end{algorithmic}
\end{algorithm}

\subsubsection{SAC-MVE (MPI)}
SAC-MVE trains the value function on the data sampled in the environment and the imagined data sampled in a learned model but only performs policy improvement on real samples. By contrast, SAC-MVE (MPI) performs policy improvement on both real and imagined samples, so we abbreviate it to MPI (Mixed Policy Improvement) for short. See Algorithm \ref{sac-mve-mpi}.

\begin{algorithm}[h]
\caption{SAC-MVE (MPI)}
\label{sac-mve-mpi}
\textbf{Input}: Initial dynamics model $p_\theta(s^\prime,r|s,a)$, policy $\pi_\phi(a|s)$, action-value $Q_\psi(s,a)$ and its target network $Q_{\psi^-}(s,a)$, rollout length $k$, environment buffer $\mathcal{D}_{\mathrm{env}}$, start size $U$, batch size $B$.

\begin{algorithmic}[1] 
\STATE Explore in the environment for $U$ steps and add data to $\mathcal{D}_{\mathrm{env}}$
\FOR{$N$ epochs}
\STATE Train model $p_\theta$ on $\mathcal{D}_{\mathrm{env}}$ by maximizing Eq.\eqref{pobject}
\FOR{$E$ steps}
\STATE Sample action to perform in the environment according to $\pi_\phi$; add the environmental transition to $\mathcal{D}_{\mathrm{env}}$
\STATE Sample transitions $\tau_{real}$ from $\mathcal{D}_{\mathrm{env}}$ and make $k$-step rollouts to get $\tau_{fake}$
\STATE Update critic $Q_\psi$ on \{$\tau_{real} \cup \tau_{fake}$\} by minimizing MVE \cite{mve} error
\STATE Update target critic via $\psi^- \leftarrow \tau\psi+(1-\tau)\psi^-$
\STATE Update policy $\pi_\phi$ on \{$\tau_{real} \cup \tau_{fake}$\} via policy gradients given by critic $Q_\psi$
\ENDFOR
\ENDFOR
\end{algorithmic}
\end{algorithm}

We must explain why we apply the MVE technique to the algorithms for comparison. First, we note that plan value estimation in our method may enable faster value function learning because it has some high-level similarities to multi-step RL methods. In order to directly verify the effectiveness of the model-based planning policy improvement, we add the MVE technique to the algorithms for comparison to make them and our method have the same learning speed of the value function as far as possible. Second, our method and MVE are similar in model utilization, which facilitates our comparison. Specifically, both of them use the dynamics model to expand several steps on the real states and only use the value estimation on the real states when performing policy improvement. However, our method can cover more states in policy improvement, as shown in Equation \eqref{mspg}, which may improve the sample efficiency. In contrast, the only way that MVE can cover more states in policy improvement is to use the value estimation on the fake states, namely SAC-MVE (MPI). What we want to emphasize is that for traditional methods, it is contradictory to avoid misleading impacts of value estimation on the fake states (SAC-MVE) and cover more states in policy improvement (SAC-MVE (MPI)). The following experimental results will show that our approach balances these two contradictions and can achieve better performance.

Figure \ref{fig6(a)} shows the performance of the three algorithms on the Hopper task. Unsurprisingly, we observe that SAC-MPPVE yields significant improvements in sample efficiency and performance. To directly show that their main difference comes from the different ways of policy improvement, rather than the quality and learning speed of the value function, we quantitatively analyze the value estimation of these three algorithms. Please note that value estimation in SAC-MVE and SAC-MVE (MPI) is identical, so we only select SAC-MPPSE and SAC-MVE for analysis of value estimation without loss of rigor. We first conduct policy evaluation for the same fixed policy on them and plot the learning curves in Figure \ref{fig6(b)}. The speed of policy evaluation of SAC-MPPVE and SAC-MVE is almost the same, and even SAC-MVE is a little faster than SAC-MPPSE. In addition, estimation quality is a key factor affecting performance. To quantitatively analyze the estimation quality, we then plot the average and standard deviation of their value estimation biases in Figure \ref{fig6(c)} and \ref{fig6(d)}. We remark that SAC-MPPVE and SAC-MVE (SAC-MVE (MPI)) have almost identical biases in value estimation, although the bias of SAC-MPPVE is somewhat larger in the early stage because of a larger action space for multi-step plan value function to fit. Summarizing these results, we can find that the value estimation of SAC-MPPVE and SAC-MVE (SAC-MVE (MPI)) are almost the same in learning speed and quality, but the performance of SAC-MPPVE is significantly better than that of SAC-Rollout-MVE. Therefore, we can conclude that our proposed model-based planning policy improvement is effective.

\end{document}